\documentclass{amsart}

\usepackage[english]{babel}
\usepackage{enumerate}
\usepackage{amsthm,amssymb}
\usepackage{float}
\usepackage{xcolor}
\usepackage{mathrsfs}
\usepackage{cases}
\usepackage{resizegather}

\usepackage{multirow}
\usepackage{amsmath}
\usepackage{graphicx}
\usepackage{amssymb}
\usepackage{mathtools}
\usepackage{url}
\newtheorem{theorem}{Theorem}[section]

\newtheorem{definition}[theorem]{Definition}
\theoremstyle{definition}

\newtheorem{exmp}[theorem]{Example}
\newtheorem{alg}[theorem]{Algorithm}

\renewcommand{\theequation}{\thesection.\arabic{equation}}
\renewcommand{\theenumi}{\roman{enumi}}
\renewcommand{\labelenumi}{\theenumi}
\newcommand{\rvline}{\hspace*{-\arraycolsep}\vline\hspace*{-\arraycolsep}}
\numberwithin{equation}{section}

\def\calF{\mathcal{F}}
\def\calC{\mathcal{C}}
\def\calM{\mathcal{M}}

\def\bbR{\mathbb{R}}

\def\bold{\mathbf}
\def\bfp{\mathbf{p}}
\def\bfv{\mathbf{v}}

\def\bfy{\mathbf{y}}
\def\bfz{\mathbf{z}}
\def\bfu{\mathbf{u}}

\def\T{\mathrm{T}}

\title[Higher Order Correlation Analysis for MVL]{Higher Order Correlation Analysis for Multi-View Learning}
\author[J. Nie]{Jiawang Nie}
\author[L. Wang] {Li Wang}
\author[Z. Zheng]{Zequn Zheng}
\thanks{
Jiawang Nie and Zequn Zheng are from
Department of Mathematics, University of California San Diego,
9500 Gilman Drive, La Jolla, CA, USA, 92093. Email: njw@math.ucsd.edu, zez084@ucsd.edu.\\
Li Wang is from the Department of Mathematics, University of Texas at Arlington,
411 South Nedderman Drive, Arlington, TX, 76019. Email:  li.wang@uta.edu.
}

\subjclass[2010]{15A69,62H30,62H35,68T10,68T30}

\begin{document}
\maketitle

\begin{abstract}
Multi-view learning is frequently used in data science.
The pairwise correlation maximization is a classical approach
for exploring the consensus of multiple views.
Since the pairwise correlation is inherent for two views,
the extensions to more views can be diversified and
the intrinsic interconnections among views are generally lost.
To address this issue, we propose to maximize higher order correlations.
This can be formulated as a low rank approximation problem
with the higher order correlation tensor of multi-view data.
We use the generating polynomial method to solve
the low rank approximation problem.
Numerical results on real multi-view data demonstrate that
this method consistently outperforms prior existing methods.
\end{abstract}

\section{Introduction}Multi-view learning is a frequently used paradigm for multi-view data,
which has broad applications.
Generally, multi-view data contains sets of samples,
each of which is depicted by a different characteristic.
For instance, an image can be described by different feature descriptors
such as color, texture and shape.
A web page contains text and images, as well as hyperlinks to other web pages.
Due to heterogeneous features extracted from each view,
multi-view learning becomes popular in data science
to reduce the heterogeneous gap among multiple views
by maximizing the consensus of multiple views in some common latent space.

There are various multi-view learning methods in the prior work.
Among them, the canonical correlation analysis (CCA),
originally introduced for measuring the linear correlation
between two sets of variables \cite{harold1936relations},
has been the workhorse for learning a common latent space
between two views \cite{yang2019survey}.
It is extended to various learning scenarios,
such as multiple views \cite{luo2015tensor,nielsen2002multiset},
nonlinear and sparse  representations \cite{andrew2013deep,hardoon2011sparse}.
Its importance has been well demonstrated in many scientific domains \cite{uurtio2017tutorial}.
As a measurement, correlation is usually defined for two sets of variables.
The extension from two to more sets can be diversified.
See \cite{nielsen2002multiset} for various combinations of objectives and constraints.
Pairwise correlation is a common criterion for
capturing the intrinsic interconnections of two views.
But for more than two views, the intrinsic interconnections
among all views are lost.

To overcome the above issue of pairwise correlations,
higher order tensor correlation methods are generally used.
They directly model interconnections as tensors.
The tensor canonical correlation analysis (TCCA)
method is introduced in \cite{luo2015tensor}
for maximizing the higher order tensor correlation.
It not only generalizes the correlation between two views
but also explores higher order correlations for more views.
The maximization of higher order tensor correlations often use
the alternating least squares (ALS) method
\cite{comon2009tensor,KoBa09}.
It is suboptimal for solving the best rank-$r$
tensor approximation problem \cite{luo2015tensor}.
This is because the set of tensors whose ranks are less than or equal to
$r$ is usually not closed. The ALS is convenient for implementation,
but its performance is generally not reliable.
When $r = 1$, the problem is reduced to the best rank-1 approximation.
Frequently used methods are higher order power iterations \cite{de2000best},
semidefinite relaxations \cite{CDN14,nie2014semidefinite},
and SVD-based algorithms \cite{guan2018convergence}.
For a generic tensor $\mathcal{F}$, the best rank-1 approximation is unique \cite{friedland2014number}. For $r>1$, there exist various methods
for computing rank-$r$ approximations, see \cite{comon2009tensor,phan2013fast,SvBdL13}.
Many of these methods are based on ALS.
Their performance is not very reliable.
Generally, only critical points can be guaranteed. We refer to \cite{comon2009tensor,comon2011sparse,LRSTA17}
for recent work on low rank tensor approximations.

In this paper, we propose a new method
for solving the higher order tensor correlation maximization problem.
The generating polynomial method is introduced to
compute low rank approximating tensors with promising performance from
the higher order correlation tensor of multi-view input data.
Consequently, the proposed method can achieve better performance than
earlier methods based on the ALS, since a good initial point
can be found by the generating polynomial method.
The proposed method is tested on two real data sets
for multi-view feature extraction.
The computational results show that our proposed method
consistently outperforms the prior existing methods.

The paper is organized as follows.
Section~\ref{sc:pre} gives some preliminaries about tensor computations.
We introduce the generating polynomial method
for low rank tensor approximation in Section~\ref{sc:GP}.
The formulation of low rank tensor approximation for
the higher order tensor correlation maximization problem
is given in Section~\ref{sc:LRTA}.
An algorithm for solving the formulated problem
is given in Section~\ref{sc:alg}.
The numerical experiments on two real multi-view data sets
are given in Section~\ref{sc:num}.

\section{Preliminary}
\label{sc:pre}

\subsection*{Notation}
The symbol $\mathbb{N}$ (resp., $\mathbb{R}$, $\mathbb{C}$)
denotes the set of nonnegative integers (resp., real, complex numbers).
For an integer $r>0$, denote the set $[r] \coloneqq \{1, \ldots, r\}$.
Uppercase letters (e.g., $A$) denote matrices,
$(A)_{i,j}$ denotes the $(i,j)$th entry of the matrix $A$,
and Curl letters (e.g., $\mathcal{F}$) denote tensors.
For a complex matrix $A$, $A^T$ denotes its transpose
and $A^*$ denotes its conjugate transpose.
The $col(A)$ denotes the column space of $A$.
Bold lower case letters (e.g., $\bold{v}$) denote vectors,
and $(\bold{v})_i$ denotes the $i$th entry of $\bold{v}$.
The $\mbox{diag}(\bold{v})$ denotes the square diagonal matrix whose
diagonal entries are given by the entries of $\bold{v}$.
For a matrix $A$, the subscript notation $A_{:,j}$ and $A_{i,:}$
respectively denotes its $j$th column and $i$th row.
For a vector $\bold{v}$, the subscript $\bold{v}_{s:t}$
denotes the subvector of $\bold{v}$ whose label is from $s$ to $t$.
Similar subscript notation is used for tensors.

Let $\mathbb{F}$ be a field (either the real field $\mathbb{R}$
or the complex field $\mathbb{C}$).
Let $m$ and $n_1, \ldots, n_m$ be positive integers.
A tensor of order $m$ and dimension $(n_1,\ldots,n_m)$
can be represented by an array $\calF$ that is labelled
by an integeral tuple $(i_1,\ldots,i_m)$,
with $1 \leq i_j \leq n_j,\, j=1,\ldots, m$, such that
\begin{align}
	\calF = (\calF_{i_1,\ldots,i_m})_{1\leq i_1 \leq n_1, \ldots, 1 \leq i_m \leq n_m}.
\end{align}
The space of all such tensors with entries in the field $\mathbb{F}$ is denoted as $\mathbb{F}^{n_1\times\cdots\times n_m}$.
The integer $m$ is the order of $\calF$.
The Hilbert-Schmidt norm of $\calF$ is
\[
\| \calF \| = \sqrt{
	\sum_{\substack{ 1\leq i_j \leq n_j, 1 \le j \le m } }
	|\calF_{i_1,\ldots,i_m}|^2
} .
\]

For vectors $\bfv_1 \in \mathbb{F}^{n_1}, \ldots, \bfv_m \in \mathbb{F}^{n_m}$,
their outer product $\bfv_1 \otimes \ldots \otimes \bfv_m$
is the tensor in $\mathbb{F}^{n_1\times\cdots\times n_m}$ such that
\[
(\bfv_1 \otimes \cdots \otimes \bfv_m)_{i_1,\ldots,i_m} =
(\bfv_1)_{i_1} \cdots (\bfv_m)_{i_m}
\]
for all labels $i_1, \ldots, i_m$ in the range.
A tensor in the form $\bfv_1 \otimes \ldots \otimes \bfv_m$
is called a rank-$1$ tensor. For each
$\calF \in \mathbb{F}^{n_1\times\cdots\times n_m}$,
there exist tuples of vectors $(\bfv^{s,1}, \ldots, \bfv^{s,m})$,
$s=1,\ldots,r$, with $\bfv^{s,j} \in \mathbb{F}^{n_j}$, such that
\begin{align}   \label{eq:ransum_pure}
	\calF = \sum_{s=1}^r \bfv^{s,1} \otimes \cdots \otimes \bfv^{s, m}.
\end{align}
The smallest such $r$ is the rank of $\calF$ over the field $\mathbb{F}$,
for which we denote $\text{rank}_{\mathbb{F}}(\mathcal{F})$.
If $\text{rank}_{\mathbb{F}}(\mathcal{F}) = r$,
the equation (\ref{eq:ransum_pure}) is called a rank-$r$ decomposition.
In the literature, $\text{rank}_{\mathbb{C}}(\mathcal{F})$
is also called the candecomp-parafac (CP) rank of $\mathcal{F}$
and (\ref{eq:ransum_pure}) is called a CP decomposition.
We refer to \cite{Land12,Lim13} for tensor theory and refer to \cite{BreVan18,dLa06,dLMV04,Domanov2014,larsen2020practical,GPSTD,SvBdL13,tensorlab}
for tensor decomposition methods.
Recent applications of tensor decompositions can be found in
\cite{guo2021learning,KoBa09,nie2020hermitian}.
Tensors are closely related to polynomial optimization
\cite{CDN14,FNZ18,nie2014semidefinite,NYZ18}.

Tensors can be naturally used to characterize multidimensional data in applications,
such as 3D images, panel data (subjects $\times$
variables $\times$ time $\times$ location),
including multi-channel EEG and fMRI data in
Neuroscience \cite{dao2020multi,davidson2013network},
higher order multivariate portfolio moments \cite{brandi2020unveil},
and multi-view datasets \cite{luo2015tensor}.
The traditional data analysis approach based on representations
by vectors or matrices has to reshape multidimensional data
into the vector/matrix format. However, such a transformation not only destroys
the intrinsic interconnections between the data points,
but also gives exponentially growing number of estimated parameters.

A tensor decomposition can be represented by matrices.
If a tensor $\mathcal{F}$ has the decomposition
\[
\mathcal{F} = \sum_{s=1}^{r}
\bold{u}^{s,1} \otimes \cdots \otimes \bold{u}^{s,m},
\]
we can denote the matrices
\[
U^{(i)} \, = \, [\bold{u}^{1,i},...,\bold{u}^{r,i}],~
i =1, \ldots,m.
\]
We call such $U^{(i)}$ the $i$th decomposing matrix for $\mathcal{F}$.
For convenience of notation, we denote that
\begin{align*}
	U^{(1)} \circ \cdots \circ U^{(m)} =
	\sum_{i=1}^r (U^{(1)})_{:,i} \otimes\ldots \otimes (U^{(m)})_{:,i}.
\end{align*}
In the above, $(U^{(m)})_{:,i}$ stands for the $i$th column of $U^{(m)}$.
For two matrices $A$ and $B$, with $A = (A_{ij}) \in \mathbb{F}^{k \times n}$
and $B = [\bold{b}_1,\ldots,\bold{b}_n ] \in \mathbb{F}^{p \times n}$,
their Khatri-Rao product $\odot$ is the matrix
\[
A \odot B \coloneqq
\begin{bmatrix}
	A_{11}\bold{b}_1 & \dots & A_{1n}\bold{b}_n \\
	\vdots & \ddots & \vdots \\
	A_{k1}\bold{b}_1 & \dots & A_{kn}\bold{b}_n \\
\end{bmatrix}.
\]
For a given matrix $V \in \mathbb{C}^{p\times n_t}$,
we define the matrix-tensor product
\[
\hat{\mathcal{F}}  \coloneqq   V \times_t \mathcal{F}  \in
\mathbb{C}^{n_1\times ...\times n_{t-1} \times p \times n_{t+1} \times ...\times n_m}
\]
such that the $i$th slice of $\hat{\mathcal{F}}$ is
\[
\hat{\mathcal{F}}_{i_1,...,i_{t-1},:,i_{t+1},...,i_m} = V \mathcal{F}_{i_1,...,i_{t-1},:,i_{t+1},...,i_m}.
\]
In particular, if $V$ is a vector $\mathbf{v} \in \mathbb{F}^{n_t}$,
then the vector-tensor product
\[
\bfv^T  \times_t  \mathcal{F}
\in \mathbb{C}^{n_1\times ...\times n_{t-1} \times 1 \times n_{t+1} \times ...\times n_m}
\]
is similarly defined. Note that the order of
$\bfv^T  \times_t  \mathcal{F}$ drops by one.

The low rank tensor approximation (LRTA) problem is
to approximate a given tensor by a low rank one.
The LRTA is equivalent to solving a nonlinear least square problem.
For a given tensor $\mathcal{F}\in  \mathbb{F}^{n_1\times\cdots\times n_m}$,
and a given rank $r$, the LRTA is to find $r$ tuples
\[
\bfv^{(s)} \coloneqq (\bfv^{s,1}, \ldots, \bfv^{s,m}) \in
\mathbb{F}^{n_1}\times \cdots\times \mathbb{F}^{n_m},
\quad s=1, \ldots, r,
\]
which gives a minimizer to the following nonlinear least square problem
\begin{align}
	\min\limits_{{\bfv^{(1)},\cdots,\bfv^{(r)}} }\big\|\mathcal{F} -
	\sum\limits_{s = 1}^r \bfv^{s,1} \otimes \cdots \otimes \bfv^{s, m}\big\|^2. \label{nonlinear:least:square:F}
\end{align}

\section{Generating Polynomials}
\label{sc:GP}

This section shows how to use generating polynomials
to compute tensor decompositions. Without loss of generality,
we assume the tensor dimensions are decreasing:
\[
n_1\geq n_2 \geq \ldots \geq n_m.
\]
We consider tensors with rank $r\leq n_1$.
Denote indeterminate variables
\[
\bold{x_1} =(x_{1,1},...x_{1,n_1}),\,\,
\bold{x_2}=(x_{2,1},...x_{2,n_2}), \,\, \ldots,
\bold{x_m}=(x_{m,1},...x_{m,n_m}).
\]
The $(i_1,i_2,...,i_m)$th entry of a tensor $\mathcal{F}$
can be labelled by a monomial $x_{1,i_1}x_{2,i_2}...x_{m,i_m}$. Let
\begin{equation}
	\begin{array}{rcl}
		\mathbb{M} & \coloneqq &
		\big\{ x_{1,i_1}...x_{m,i_m} \; | \;1\leq i_j \leq n_j, 1 \leq j \leq m \big\}, \\
		\mathcal{M} & \coloneqq & {\rm span}\{ \mathbb{M} \} .
	\end{array}
\end{equation}
For a subset $J \subseteq \{1,2,...,m \}$, we denote that
\begin{equation} 	\label{m_j_def}
	\begin{array}{rcl}
		J^c  & \coloneqq & \{1,2,...,m \} \backslash J,  \\
		\mathbb{M}_J  & \coloneqq & \big\{ x_{1,i_1}...x_{m,i_m} \; | \;
		{ x_{j,i_j}=1}, \;  \;  j\in J^c \big\}, \\
		\mathcal{M}_J & \coloneqq & {\rm span} \{\mathbb{M}_J\}.
	\end{array}
\end{equation}
Note that $(i_1,\ldots,i_m)$ is uniquely determined by the monomial
$x_{1,i_1} \cdots x_{m,i_m}$.
So a tensor $\mathcal{F}\in \mathbb{C}^{n_1\times \ldots \times n_m}$
can be equivalently labelled as
\begin{align} \label{label:monomial}
	\mathcal{F}_{x_{1,i_1}\ldots x_{m,i_m}}   \coloneqq   \mathcal{F}_{i_1,\ldots,i_m}.
\end{align}
With the above new labelling, we define the bi-linear operation
$\langle \cdot, \cdot \rangle$
between $\mathcal{M}_J$ and $\mathbb{C}^{n_1,\ldots,n_m}$ as
\begin{align}  	\label{polyten}
	\langle \sum_{\mu \in \mathbb{M}}c_{\mu}\mu,\mathcal{F}\rangle
	\coloneqq  \sum_{\mu \in \mathbb{M}}c_{\mu}\mathcal{F}_{\mu}.
\end{align}
In the above, each $c_{\mu}$ is a scalar
and $\mathcal{F}$ is labelled by monomials as in (\ref{label:monomial}).

\begin{definition}\label{def:gp}
	For a subset $J \subseteq \{1,2,...,m\}$ and a tensor
	$\mathcal{F} \in \mathbb{C}^{n_1 \times \cdots \times n_m}$,
	a polynomial $p \in \mathcal{M}_J$
	is called a generating polynomial for $\mathcal{F}$ if
	\begin{equation}  		\label{inner_p}
		\langle pq,\mathcal{F} \rangle =0 \quad
		\mbox{for all} \, \, q \in \mathbb{M}_{J^c} .
	\end{equation}
\end{definition}

The following is an example of generating polynomials.

\begin{exmp}
	Consider the cubic order tensor
	$\mathcal{F} \in \mathbb{C}^{3 \times 3 \times 3}$ given as
	\begin{gather*}
		\begin{bmatrix}
			\mathcal{F}_{:,:,1} & \rvline & \mathcal{F}_{:,:,2} & \rvline & \mathcal{F}_{:,:,3} \\
		\end{bmatrix}=
		\begin{bmatrix}
			\begin{matrix}-10 & 48 & 70\\
				-10 & -64 & -50 \\
				-5 & 10 & 20 \\
			\end{matrix} & \rvline & \begin{matrix}
				22 & -16 & -58 \\
				-42 & 0 & 78 \\
				3 & -6 & -12 \\\end{matrix} &\rvline & \begin{matrix} -1 & 44 & 49 \\
				-29 & -68 & -19 \\
				-4 & 8 & 16 \end{matrix} \\
		\end{bmatrix}.
	\end{gather*}
	{For $J=\{ 1,2 \}$,
		note that $\begin{bmatrix}6 & 3 & 2 & 1 \end{bmatrix}$ is orthogonal to
		\[
		\begin{bmatrix}
			\mathcal{F}_{1,1,i_3} & \mathcal{F}_{1,2,i_3} &
			\mathcal{F}_{2,1,i_3} & \mathcal{F}_{2,2,i_3}
		\end{bmatrix}
		\]
		for $i_3=1,2,3$.
		So $\begin{bmatrix}6 & 3 & 2 & 1 \end{bmatrix}$
		is the coefficient vector of a generating polynomial.}
	The following is a generating polynomial for $\mathcal{F}$:
	{
		\[
		p \coloneqq (3x_{1,1}+x_{1,2})(2x_{2,1}+x_{2,2}).
		\]}Note that $p \in \mathcal{M}_{\{1,2 \}}$
	and for each $i_3 = 1, 2, 3$
	\begin{align*}
		p \cdot x_{3,i_3}  =  (3x_{1,1}+x_{1,2})(2x_{2,1}+x_{2,2})  x_{3,i_3}.
	\end{align*}
	One can check that for each  $i_3 = 1, 2, 3$
	\[
	6\mathcal{F}_{1,1,i_3}+3\mathcal{F}_{1,2,i_3}+
	2\mathcal{F}_{2,1,i_3}+\mathcal{F}_{2,2,i_3}=0.
	\]
	So, $\langle pq,\mathcal{F} \rangle =0$
	for all $q \in \mathcal{M}_{\{3\}}$,
	hence $p$ is a generating polynomial.
\end{exmp}

Suppose the rank $r \le n_1$ is given.
For convenience of notation, denote the label set
\begin{align}  \label{set:J}
	J  \coloneqq  \{(i,j,k):1\leq i \leq r, ~2 \leq j \leq m, ~2\leq k \leq n_j \}.
\end{align}
For a matrix $G \in \mathbb{C}^{[r]\times J}$
and a triple $\tau =(i,j,k) \in J$, define the bi-linear polynomial
\begin{align}  	\label{phi_def}
	\phi[G,\tau](x) \coloneqq  \sum_{\ell=1}^{r}
	G(\ell,\tau)x_{1,\ell}x_{j,1}-x_{1,i}x_{j,k}
	\,\, \in \, \mathcal{M}_{\{1,j\}}.
\end{align}
The rows of $G$ are labelled by $\ell=1,2,...,r$
and the columns of $G$ are labelled by $\tau \in J$.
We are interested in $G$ such that
$\phi[G,\tau]$ is a generating polynomial for a tensor
$\mathcal{F} \in \mathbb{C}^{n_1 \times n_2 \times \ldots \times n_m}$.
This requires that
\begin{align*}
	\langle \phi[G,\tau] \cdot \mu ,\mathcal{F} \rangle=0 \,
	\quad \mbox{for all} \, \, \mu \in \mathbb{M}_{\{1,j\}^c}.
\end{align*}
The above is equivalent to the equation
($\mathcal{F}$ is labelled as in (\ref{label:monomial}))
\begin{equation}  	\label{linear_eq_g}
	\sum_{\ell=1}^{r} G(\ell,\tau)\mathcal{F}_{x_{1,\ell} \cdot \mu}
	=\mathcal{F}_{x_{1,i}  x_{j,k}  \cdot \mu} .
\end{equation}

\begin{definition}
	If \eqref{linear_eq_g} holds for all $\tau \in J$,
	then $G$ is called a generating matrix for $\mathcal{F}$.
\end{definition}

For given $G$, $j \in \{ 2, \ldots, m \}$ and $k \in \{2,\ldots,n_j \}$,
we denote the matrix
\begin{equation}
	M^{j,k}[G] \coloneqq \begin{bmatrix}
		G(1,(1,j,k)) & G(2,(1,j,k)) & \dots & G(r,(1,j,k)) \\
		G(1,(2,j,k)) & G(2,(2,j,k)) & \dots & G(r,(2,j,k)) \\
		\vdots & \vdots & \ddots & \vdots \\
		G(1,(r,j,k)) & G(2,(r,j,k)) & \dots & G(r,(r,j,k)) \\
	\end{bmatrix} .
	\label{mjk}
\end{equation}
For each $j,k$, define the matrices
\begin{equation}
	\label{def:A_b}
	\begin{cases}
		~~A[\mathcal{F},j] ~~~~~   \coloneqq &
		\Big(\mathcal{F}_{x_{1,\ell}\cdot \mu}
		\Big)_{\mu \in \mathbb{M}_{\{1,j\}^c}, 1\leq \ell \leq r} , \\
		~~B[\mathcal{F},j,k] ~~ \coloneqq &
		\Big(\mathcal{F}_{x_{1,\ell} \cdot x_{j,k}\cdot \mu}
		\Big)_{\mu \in \mathbb{M}_{\{1,j\}^c}, 1\leq \ell \leq r} .	
	\end{cases}
\end{equation}
Then the equation \eqref{linear_eq_g} is equivalent to
\begin{align}   	\label{linear_eq}
	A[\mathcal{F},j] (M^{j,k}[G])^T = B[\mathcal{F},j,k].
\end{align}
The following is a useful property
for the matrices $M^{j,k}[G]$.
\begin{theorem}  \label{decom_to_Mjk}
	Suppose $\mathcal{F}=\sum_{s=1}^r
	\bold{u}^{s,1} \otimes ... \otimes \bold{u}^{s,m}$
	with vectors $\bold{u}^{s,j}\in \mathbb{C}^{n_j}$.
	If $r \leq n_1$, $(\bold{u}^{s,2})_1...(\bold{u}^{s,m})_1 \neq 0$,
	and the first $r$ rows of the first decomposing matrix
	\[
	U^{(1)} \coloneqq [\bold{u}^{1,1} \, \, \cdots \,\,  \bold{u}^{r,1} ]
	\]
	are linearly independent,
	then there exists a $G$ satisfying \eqref{linear_eq}
	and satisfying (for all $j \in \{ 2, \ldots, m \}$,
	$k \in \{2,\ldots,n_j \}$ and $s=1,\ldots,r$)
	\begin{align}  	\label{commuting_eq}
		M^{j,k}[G]\cdot (\bold{u}^{s,1})_{1:r}=
		(\bold{u}^{s,j})_k\cdot (\bold{u}^{s,1})_{1:r}.
	\end{align}
\end{theorem}
\begin{proof}
	Since $(\bold{u}^{s,2})_1...(\bold{u}^{s,m})_1 \neq 0$,
	we can generally assume
	\[
	(\bold{u}^{s,2})_1 \cdots (\bold{u}^{s,m})_1 = 1,
	\]
	up to a scaling on $\bold{u}^{s,1}$. Denote the matrices
	\begin{align*}
		\hat{U}_{1}= &
		\begin{bmatrix}
			(\bold{u}^{1,1})_{1:r} & (\bold{u}^{2,1})_{1:r} & \dots & (\bold{u}^{r,1})_{1:r} \\
		\end{bmatrix}, \\
		\hat{U}_{j}= &
		\begin{bmatrix}
			\bold{u}^{1,j} & \bold{u}^{2,j} & \ldots & \bold{u}^{r,j}\\
		\end{bmatrix},~~~~j = 2,\ldots,m.
	\end{align*}
	Since $r \leq n_1$ and the first $r$ rows of $U^{(1)}$
	are linearly independent, the matrix $\hat{U}_{1}$ is invertible.
	Let $U^{(j)}$ be the $j$th decomposing matrix of $\mathcal{F}$.
	For $j=2,\ldots m,\, k=2,\ldots,n_j$, denote
	\[
	W_{j}  \coloneqq  U^{(2)}\odot  \cdots \odot U^{(j-1)}\odot U^{(j+1)}
	\cdots \odot U^{(m)},
	\]
	\[
	\Lambda_{j,k} \coloneqq \mbox{diag}((U^{(j)})_{k,:}) .
	\]
	Then one can verify that
	\begin{align} \label{proof1:ab}
		A[\mathcal{F},j]=W_{j} \Lambda_{j,1} \hat{U}_{1}^T,\quad
		B[\mathcal{F},j,k]=W_{j} \Lambda_{j,k} \hat{U}_{1}^T.
	\end{align}
	Let $\hat{G}$ be the matrix such that
	for all $(i,j,k) \in J$
	\begin{align} \label{proof1:mjk}
		M_{j,k}[\hat{G}] \coloneqq \hat{U}_{1} \Lambda_{j,k} (\hat{U}_{1})^{-1}.
	\end{align}
	Note that each $M_{j,k}[\hat{G}]$ satisfies \eqref{commuting_eq}.
	We next show that $\hat{G}$ is a generating matrix for $\mathcal{F}$.
	Applying expressions in \eqref{proof1:ab} and
	\eqref{proof1:mjk} to \eqref{linear_eq}, we get that
	{\begin{align*}
			A[\mathcal{F},j] (M^{j,k}[\hat{G}])^T
			&= W_j \Lambda_{j,1} \hat{U}_{1}^T({{\hat{U}_{1}}^{T}})^{-1} \Lambda_{j,k} (\hat{U}_{1})^T \\
			&=W_j \Lambda_{j,1}\Lambda_{j,k} \hat{U}_{1}^T=B[\mathcal{F},j,k] .
	\end{align*}}
	The above implies that $M^{j,k}[\hat{G}]$
	satisfies \eqref{linear_eq} for all $j,k$ in the range,
	so $\hat{G}$ is a generating matrix for $\mathcal{F}$.
\end{proof}

Theorem~\ref{decom_to_Mjk} implies that if the tensor
$\mathcal{F}$ has rank $r \leq n_1$ and has generic decomposing vectors,
there exists a generating matrix $G$ such that all $M^{j,k}[G]$
are simultaneously diagonalizable as in (\ref{commuting_eq}).
That is, there exists an invertible matrix
$V=[\bold{v}_1,\bold{v}_2,\ldots,\bold{v}_r]$ such that
\[
V^{-1}M^{j,k}[G]V = \mbox{diag}
[\lambda_{j,k,1},\lambda_{j,k,2},\ldots,\lambda_{j,k,r}]
\]
are all diagonal. For this case,
there must exist scalars $c_1,c_2,\ldots,c_r$ such that
\begin{equation}  	\label{proof}
	\mathcal{F}_{1:r,1,...,1}  =
	c_1 \bold{v}_1+c_2 \bold{v}_2+...+c_r \bold{v}_r .
\end{equation}
Let $\hat{\mathcal{F}} \coloneqq \mathcal{F}_{1:r,:,\cdots,:}$
be the subtensor and let
\[
\mathcal{H} \coloneqq
c_1  \bold{w}^{1,1}\otimes \cdots \otimes \bold{w}^{1,m}+ \cdots+
c_r  \bold{w}^{r,1}\otimes \cdots \otimes \bold{w}^{r,m} ,
\]
where the vectors $\bold{w}^{s,1}=\bold{v}_s$
and ($1\leq s \leq r,  2\leq j \leq m$)
\[
\bold{w}^{s,j} \coloneqq
\begin{bmatrix}
	1 & \lambda_{j,2,s} & \lambda_{j,3,s} & \cdots & \lambda_{j,n_j,s}
\end{bmatrix}^T.
\]
Then we show that $\mathcal{H} = \hat{\mathcal{F}}$.
By \eqref{inner_p} and \eqref{phi_def},
\[
\langle \phi[G,\tau] p, \hat{\mathcal{F}} \rangle  =
\langle  \phi[G,\tau] p, \mathcal{H} \rangle  =0,
\]
for all $p \in \mathbb{M}_{\{1,j\}^c}$, so
\begin{equation}
	\langle \phi[G,\tau] p, \mathcal{H}-\hat{\mathcal{F}} \rangle  =0,
	\quad \mbox{for all} \, \, p \in \mathbb{M}_{\{1,j\}^c}.
\end{equation}
The equation \eqref{proof} implies that
\begin{align}
	(\mathcal{H}-\hat{\mathcal{F}})_{1:r,1,\ldots,1}=0. \label{first:zero}
\end{align}
{By \eqref{first:zero}}, 
for $\tau = (i,2,k) \in J$,  we get
\[
\langle \mathcal{H}-\hat{\mathcal{F}},\phi[G,\tau]\rangle  =0,~~~~
(\mathcal{H}-\hat{\mathcal{F}})_{:,:,1,\ldots,1}=0.
\]
Then, for $\tau = (i,2,k) \in J$,  we have
\[
\langle \mathcal{H}-\hat{\mathcal{F}},\phi[G,\tau] x_{2,:}\rangle  =0,~~~~ (\mathcal{H}-\hat{\mathcal{F}})_{:,:,:,1,\ldots,1}=0.
\]
Doing this inductively, we can see $\mathcal{H}=\hat{\mathcal{F}}$.
Since $\hat{\mathcal{F}}=\mathcal{F}_{1:r,:...}$
and $\mathcal{F}$ has rank $r$,
$\mathcal{F}$ has a tensor decomposition
\[
{\mathcal{F}=U^{(1)} \circ U^{(2)} \circ \ldots \circ U^{(m)} .}
\]
Let $W \coloneqq {U^{(1)}(U^{(1)}_{1:r,:})^{-1}} \in \mathbb{C}^{n_1\times r}$,
then $(W)_{1:r,:}=I_r $ and
\begin{align}
	\mathcal{F} =  W \times_1  \hat{\mathcal{F}} .
\end{align}
This implies the tensor decomposition
\begin{equation}
	{\mathcal{F} = \sum_{i=1}^r c_i \mathbf{\hat{w}}^{i,1}
		\otimes \mathbf{w}^{i,1} \otimes \cdots \otimes \mathbf{w}^{i,m},}
\end{equation}
where the vectors
\[
\mathbf{\hat{w}}^{i,1} = W \mathbf{w}^{i,1} =
\begin{bmatrix}
	\mathbf{w}^{i,1} \\ \mathbf{\tilde{w}}^{i,1}
\end{bmatrix}.
\]
In computation, we do not need to compute the matrix $W$ explicitly.
The vectors $\mathbf{\tilde{w}}^{i,1}$
can be obtained by solving the linear least squares
\begin{align}  	\label{eq:solve least_squares_2}
	\min_{ \bold{z}_1,...,\bold{z}_{r} }
	\bigg\lVert \sum_{s=1}^r \bold{z}_s\otimes \bold{w}^{s,2}
	\otimes ...\otimes \bold{w}^{s,m} -\mathcal{F}_{r+1:n_1,:,...,:}
	\bigg\rVert^2.
\end{align}
The optimal solutions are the vectors $\tilde{\mathbf{w}}^{i,1}$.
Then $\mathcal{F}$ has the rank-$r$ decomposition
\begin{align}  \label{eq:decom newf}
	\mathcal{F} = \sum_{i=1}^r \mathbf{\hat{w}}^{i,1}\otimes
	\mathbf{w}^{i,2} \otimes ... \otimes \mathbf{w}^{i,m} .
\end{align}

When $\mathcal{F}$ is a rank-$r$ tensor,
the above process can produce a rank-$r$ decomposition for $\mathcal{F}$.
When $\mathcal{F}$ is near to a rank-$r$ tensor,
one can similarly obtain a rank-$r$ tensor approximation for $\mathcal{F}$.
This is shown in Section~\ref{sc:alg}.

\section{The Higher order Tensor Correlation Maximization}
\label{sc:LRTA}

Let  $\{ (\bfy_{i,1}, \ldots, \bfy_{i,m}) \}_{i=1}^N$ be a multi-view data set,
with $m$ views and $N$ points.
The vector $\bfy_{i, j} \in \bbR^{n_j}$ is the $i$th data point of the view $j$
residing in the $n_j$-dimensional space.
We are looking for a $r$-dimensional latent space $\bbR^r$ such that each
$\bfy_{i, j}$ is projected to $\bfz_{i, j} \in \bbR^{r}$.
The projection for the $j$th view
can be represented by a matrix $P_j$, that is,
$\bfz_{i, j}  = P_j^{\T} \bfy_{i,j}$.
The higher order canonical correlation $\rho$ of $m$ views is the quantity
\begin{align}  \label{eq:hocc}
	\rho \coloneqq \sum_{i=1}^N \sum_{s=1}^r \prod_{j=1}^m (\bfz_{i,j})_s.
\end{align}
The tensor canonical correlation analysis aims to
find optimal projection matrices $P_1, \ldots, P_m$ that maximize $\rho$.
When $m=2$, $\rho$ reduces to the trace of sample cross-correlation,
which is used in the classical canonical correlation analysis.
When $m \geq 3$, $\rho$ generalizes CCA
for capturing higher order correlations,
which is inherently different from the sum of pairwise correlations
\cite{nielsen2002multiset}.

The connection of $\rho$ as in (\ref{eq:hocc}) to a tensor can be built based on
the $t$-mode product of a tensor obtained from the input data
with projection matrices $P_1, \ldots, P_m$.
The tensor of the input data is the $m$th order tensor
of dimension $n_1 \times \cdots \times n_m$
\begin{align}
	\calC \coloneqq \sum_{i=1}^N \bfy_{i,1} \otimes \cdots \otimes \bfy_{i,m}.
\end{align}
Write that $P_j = [\bfp^{1,j}, \ldots, \bfp^{r,j}]$,
where $\bfp^{s,j} \in \bbR^{n_j}$ is the $s$th column of $P_j$.
The higher order canonical correlation $\rho$ can be written as
\begin{align}
	\rho = \sum_{s=1}^r
	(\bfp^{s,1})^T \times_1  \cdots  (\bfp^{s,m})^T \times_m \calC.
\end{align}
People often pose the uncorrelation constraints
for projected points in the latent common space
\begin{align}
	\frac{1}{N} \sum_{i=1}^N \bfz_{i,j} \bfz_{i,j}^{\T} = P_j^{\T} C_j P_j= I_r,
	j=1,\ldots, m,
\end{align}
where the $j$th view matrix
\[
C_j  \coloneqq \frac{1}{N} \sum_{i=1}^N \bfy_{i,j} \bfy_{i,j}^{\T}
\]
Denote the vectors and tensor
\begin{align}  	\label{tcca:p}
	\bfu^{s,j}  \coloneqq  C_j^{\frac{1}{2}} \bfp^{s,j}, \quad
	\bfp^{s,j}  \coloneqq  C_j^{-\frac{1}{2}} \bfu^{s,j},
\end{align}
\begin{equation} \label{tensor:cM}
	\calM \coloneqq
	C_1^{-\frac{1}{2}} \times_1 \cdots  C_m^{-\frac{1}{2}}  \times_m \calC.
\end{equation}
Then, we get the tensor correlation maximization problem
\begin{equation}  \label{op:tcca-origin}
	\left\{\begin{array}{cl}
		\max\limits_{\bfu^{s,j}} & \sum_{s=1}^r
		(\bfu^{s,1})^T \times_1  \cdots  (\bfu^{s,m})^T \times_m  \calM \\
		\textrm{s.t.} ~&   \| \bfu^{s,j} \|_2 = 1, \,
		s= 1, \ldots r, \, j=1, \ldots, m,\\
		&(\bfu^{s,j})^{\T} \bfu^{s',j} = 0 \quad \mbox{for all} \,\, s \ne  s'.
	\end{array} \right.
\end{equation}
The above is equivalent to
the rank-$r$ tensor approximation problem
\begin{equation} \label{lra:orth}
	\left\{\begin{array}{rl}
		\min\limits_{\bfu^{s,j}, \lambda_s}  & \Big\| \calM -
		\sum_{s=1}^r \lambda_{s} \cdot \bfu^{s,1}
		\otimes \cdots \otimes \bfu^{s, m} \Big\|^2, \\
		\textrm{s.t.}  & \| \bfu^{s,j} \|_2 = 1,
		s= 1, \ldots r, \, j=1, \ldots, m, \\
		&(\bfu^{s,j})^{\T} \bfu^{s',j} = 0 \quad \mbox{for all} \,\, s \ne  s'.
	\end{array} \right.
\end{equation}
The optimization (\ref{lra:orth}) requires to compute the best
rank-$r$ orthogonal tensor approximation.
This is typically a computationally hard task.
Generally, the orthogonality constraints in (\ref{op:tcca-origin})
is hard to be enforced, because the rank decomposition and
the orthogonal decomposition are usually not
achievable simultaneously \cite{Lathauwer2000}.
For better performance in computational practice,
people often relax the orthogonality constraints (see  \cite{luo2015tensor})
and then solve the following relaxation of (\ref{lra:orth}):
\begin{equation} \label{op:lra}
	\left\{\begin{array}{rl}
		\min\limits_{\bfu^{s,j}, \lambda_s}  & \Big\| \calM -
		\sum_{s=1}^r \lambda_{s} \cdot \bfu^{s,1}
		\otimes \cdots \otimes \bfu^{s, m} \Big\|^2, \\
		\textrm{s.t.}  & \| \bfu^{s,j} \|_2 = 1,
		s= 1, \ldots r, \, j=1, \ldots, m.
	\end{array} \right.
\end{equation}
After the vectors $\bfu^{s,j}$ are obtained by solving (\ref{op:lra}),
the projection matrices $P_j$ can be chosen such that
$ \bfp^{s,j}= C_r^{-\frac{1}{2}} \bfu^{s,j}$.
We would like to remark that when $\calM$
is sufficiently close to a rank-$r$ orthogonal tensor,
the optimizer of (\ref{op:lra})
is expected to be close to a rank-$r$ orthogonal tensor.

\section{The Algorithm for TCCA}
\label{sc:alg}

For the given multi-view data set
$\{ (\bfy_{i,1}, \ldots, \bfy_{i,m}) \}_{i=1}^n$,
we can formulate the tensor $\mathcal{M}$ as in (\ref{tensor:cM}).
Then compute a low rank approximating tensor for
$\mathcal{M}$ and use it to get the projection matrices $P_j$.

We use the method described in Section~\ref{sc:GP}
to compute a rank-$r$ approximation for $\calM$.
Suppose the rank $r \leq n_1$.
By \eqref{def:A_b}, the equation \eqref{linear_eq_g} is equivalent to
\begin{align}  	\label{linear:eq2}
	A[\mathcal{M},j] (M^{j,k}[G])^T = B[\mathcal{M},j,k].
\end{align}
Due to noises, the linear equation (\ref{linear:eq2})
may be overdetermined or even inconsistent.
Therefore, we look for a matrix $G$ that satisfies
\eqref{linear:eq2} as much as possible.
This can be done by solving linear least squares.
Let $G^{ls}$ be a least square solution to
\begin{align}
	\min_{G\in \mathbb{C}^{[r] \times J} } \sum_{\tau=(i,j,k)\in J}
	\bigg\lVert A[\mathcal{M},j] M^{j,k}[G]^T - B[\mathcal{M},j,k] \bigg\rVert^2.
	\label{least_squares_1}
\end{align}
After $G^{ls}$ is obtained, select generic scalars
{ $\xi_{j,k} \in \mathbb{R}$ obeying the standard normal distribution}
and let
\begin{align}  	\label{xi_M}
	M[\xi,G^{ls}] \coloneqq \sum_{(1,j,k)\in J} \xi_{j,k}M^{j,k}[G^{ls}].
\end{align}
We can compute its Schur Decomposition as
\begin{align}  	\label{schur_decom}
	Q^*M[\xi,G^{ls}]Q=T,
\end{align}
where $Q=[\bold{q}_1,\ldots,\bold{q}_r]$ is unitary and $T$ is upper triangular.
For $s=1,...,r$, $j=2,...,m$, let
\begin{align}  	\label{compute_v}
	\bold{v}^{s,j} \coloneqq (1, \bold{q}_s^* M^{j,2}[G^{ls}]\bold{q}_s,
	\ldots ,\bold{q}_s^* M^{j,n_j}[G^{ls}]\bold{q}_s) .
\end{align}
{ If the noises are big, it may have complex eigenvalue pairs,
	$\bold{q}_s$ maybe complex and the above vectors $\bold{v}^{s,j}$ maybe complex.}
In computational practice,
we can choose the real part to get a real
low rank tensor approximation.
Denote the real part of $\bold{v}^{s,j}$ by $\bold{\hat{v}}^{s,j}_{real}$.
After they are obtained,
we solve the linear least squares problem
\begin{align}
	\min_{\bold{z}_1,...,\bold{z}_{r}\in \mathbb{R}^{n_1}} \bigg\lVert \sum_{s=1}^r \bold{z}_s\otimes \bold{v}^{s,2}_{real} \otimes \bold{v}^{s,3}_{real} \otimes ... \otimes \bold{v}^{s,m}_{real}-{\mathcal{M}}\bigg\rVert^2.
	\label{least_squares_2}
\end{align}
Let $(\bold{v}^{1,1},\bold{v}^{2,1},...,\bold{v}^{r,1})$
be optimal ones for the least squares problem \eqref{least_squares_2}.
Then we consider the tensor
\begin{align}
	\mathcal{X}^{gp} \coloneqq \sum_{s=1}^r \bold{v}^{s,1} \otimes
	\bold{v}^{s,2}_{real} \otimes ... \otimes \bold{v}^{s,m}_{real}.
	\label{constuct_solution}
\end{align}
It can be used as an initial point for
solving the nonlinear optimization
\begin{equation}   	\label{least_squares_3}
	\min_{\bold{u}^{s,j}\in \mathbb{R}^{n_j}}
	\bigg\lVert \sum_{s=1}^r \bold{u}^{s,1} \otimes \bold{u}^{s,2}
	\otimes  \cdots  \otimes \bold{u}^{s,m}-\mathcal{M}\bigg\rVert^2.
\end{equation}
By solving (\ref{least_squares_3}), one can improve
the quality of the rank-$r$ approximating tensor $\mathcal{X}^{opt}$.
Finally, we get the projection
matrices $P_1, \ldots, P_m$ as in \eqref{tcca:p}.

The above can be summarized as the following algorithm.

\begin{alg}   \label{Algorithm-TCCA}
	(A generating polynomial method for TCCA)
	\begin{itemize}
		
		\item [Input:] a multi-view data set
		$\{ (\bfy_{i,1}, \ldots, \bfy_{i,m}) \}_{i=1}^N$
		and an approximating rank $r \le n_1$.
		
		\item [Step~1.]
		Generate tensor $\mathcal{M} \in \mathbb{R}^{n_1\times \cdots n_m}$
		as in \eqref{tensor:cM}.
		
		\item [Step~2.]
		Solve the linear least squares \eqref{least_squares_1} of tensor $\mathcal{M}$ for an  optimizer $G^{ls}$.
		
		\item [Step~3.]
		Choose generic {$\xi_{j,k} \in \mathbb{R}$
			obeying the standard normal distribution}
		and formulate $M[\xi,G^{ls}]$ as in \eqref{xi_M}.
		Compute the Schur Decomposition \eqref{schur_decom}.
		
		\item [Step~4.]
		For $s\in 1,...,r$ and $j \in 2,...,m$,
		compute $\bold{v}^{s,j}$ as in \eqref{compute_v}
		and keep its real part only ($\bold{v}^{s,j}=\mathrm{real}(\bold{v}^{s,j})$).
		Solve \eqref{least_squares_2} for optimal solution $(\bold{v}^{1,1},\bold{v}^{2,1},...,\bold{v}^{r,1})$.
		
		\item [Step~5.]
		Compute an improved solution $\bold{u}^{s,j}$ as in
		\begin{equation*}
			\min_{\bold{u}^{s,j}\in \mathbb{R}^{n_j}}
			\bigg\lVert \sum_{s=1}^r \bold{u}^{s,1} \otimes \bold{u}^{s,2} \otimes ... \otimes \bold{u}^{s,m}-\mathcal{F}\bigg\rVert^2.
		\end{equation*}
		
		\item [Output:] The matrices $P_j, \ldots, P_m$ as in \eqref{tcca:p}.

	\end{itemize}
\end{alg}

When $\mathcal{F}$ is a rank-$r$ tensor,
Algorithm~\ref{Algorithm-TCCA} should give
a rank-$r$ decomposition for $\mathcal{F}$.
When $\mathcal{F}$ is close to a rank-$r$ tensor,
Algorithm~\ref{Algorithm-TCCA} is expected to give
a good rank-$r$ approximation.
An interesting future work is to study the stability analysis.

\section{Numerical Experiments on Multi-View Data}
\label{sc:num}

We implement the Algorithm~\ref{Algorithm-TCCA} in
{\tt MATLAB} and run numerical experiments in {\tt MATLAB} 2020b
on a workstation with Ubuntu 20.04.2 LTS,
Intel® Xeon(R) Gold 6248R CPU @ 3.00GHz and memory 1TB.
We evaluate our algorithm for multi-view feature extraction
by comparing it with two baseline methods on two real data sets.

\subsection{Data description and experimental setup}
\label{sec:setting}

Two image data sets are used in this experiment: Caltech101-7 \cite{fei2007learning}
and Scene15 \cite{lazebnik2006beyond}.
We applied six feature descriptors to extract features of
views including centrist \cite{wu2008place}, gist \cite{oliva2001modeling},
lbp \cite{ojala2002multiresolution},
histogram of oriented gradient (hog), color histogram (ch),
and sift-spm \cite{lazebnik2006beyond}. Note that Scene15 consists of gray images,
so ch is not used. The statistics of the two multi-view data sets
are summarized in Table~\ref{tab:datasets}.

\begin{table}[!t]
	\caption{Data sets used in the experiments.}  \label{tab:datasets}
	\centering
	\begin{tabular}{@{}c|c|c|c|c|c|c|c|c@{}}
		\hline
		Data set &  samples & class& centrist & gist &lbp &hog &ch & sift-spm\\
		\hline			
		Caltech101-7 & 1474 & 7 &  254 &  512 &  1180 &  1008 &  64 & 1000\\
		Scene15 & 4310 & 15 & 254 & 512 & 531  & 360 & - &  1000\\
		\hline
	\end{tabular}
\end{table}

As our main focus is on data sets with more than two views, our proposed algorithm is evaluated by comparing with multiset CCA (mcca) \cite{via2007learning}
and TCCA using ALS (als) \cite{luo2015tensor}.
For each data, we first apply principal component analysis (PCA) \cite{pca2002}
to each view to reduce the input dimension to 20 so that the constructed tensor
can be properly handled by tensor-based methods.
And then, we split the data into training and testing sets
with a predefined training ratio. All compared methods
are run on the training data to get the projection matrix
of each view for a given dimension of the common space (or rank).
To report the testing accuracy, we apply the learned projection matrix
to both training and testing sets of each view,
concatenate the projected features of all views
as the final representation of each sample,
train linear support vector classifier (SVC) \cite{chang2011libsvm}
on training data and evaluate the performance of the trained classifier
on testing data. The classification accuracy is used as the evaluation metric.
The regularization parameter of the linear SVC
is tuned in $\{0.01, 0.1, 1, 10, 100\}$.
The experiments of the compared methods
on each data set are repeated ten times with
randomly sampled training and testing sets,
and the mean accuracy with standard deviation
on the ten experiments are reported for compared methods.

\begin{figure}
	\centering
	\begin{tabular}{cc}
		\includegraphics[width=0.4\textwidth]{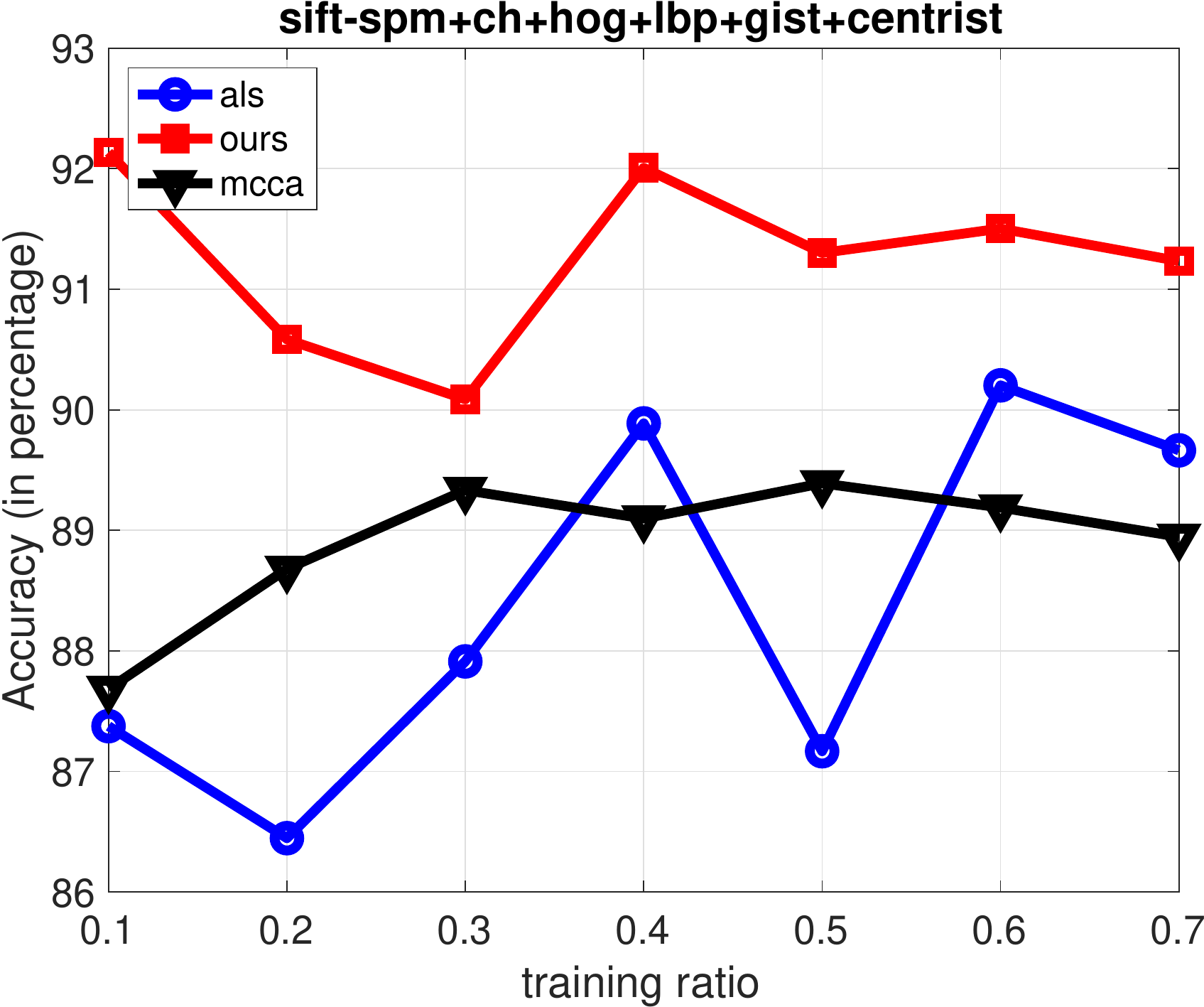} & \includegraphics[width=0.4\textwidth]{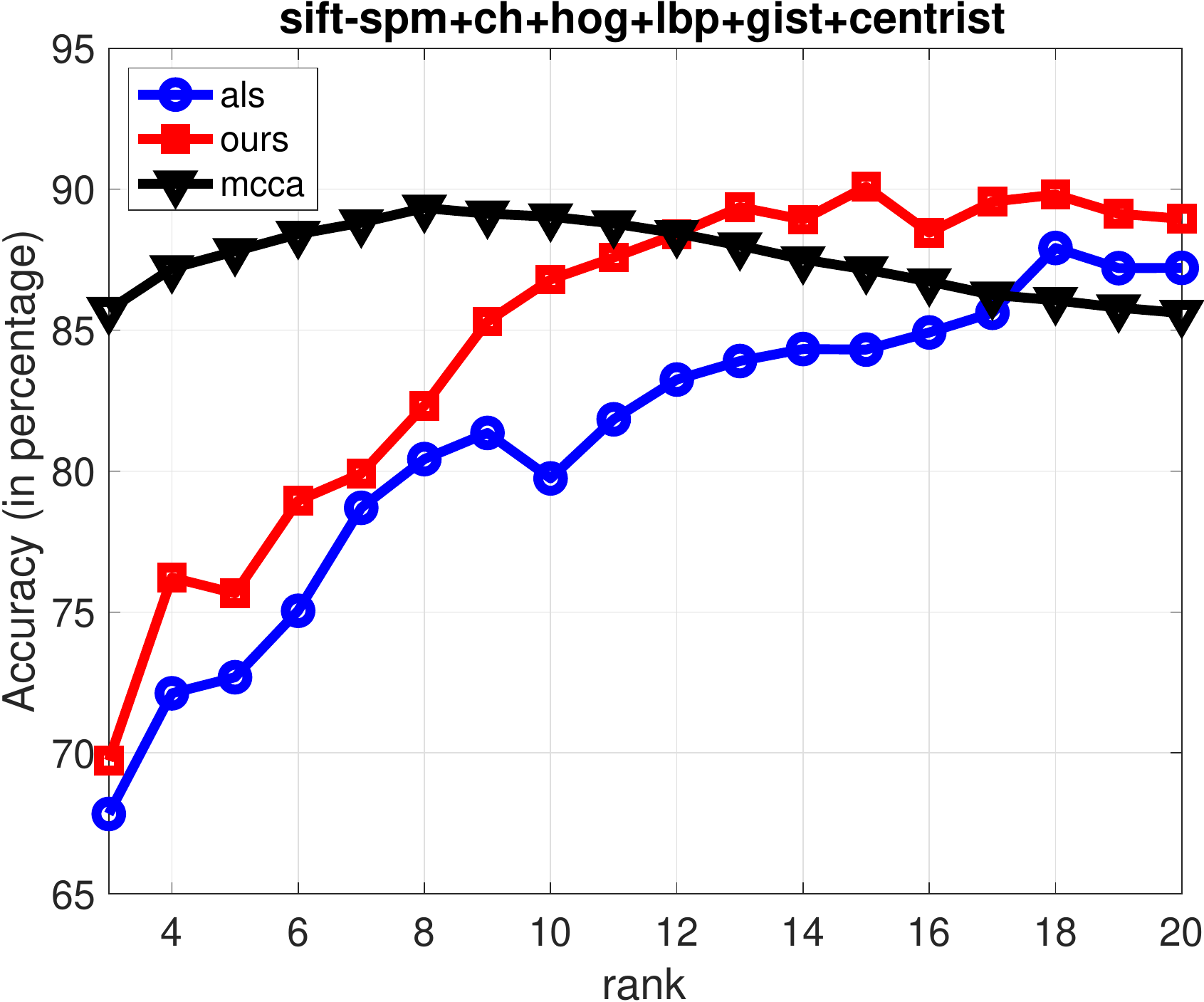}\\
		(a) & (b)
	\end{tabular}
	\caption{Sensitivity analysis of compared methods with six views on data Caltech101-7. (a) varying the training ratios over common spaces from 3 to 20; (b) varying the size of common space on 30\% training data.} \label{fig:sense-caltech}
\end{figure}

\subsection{Experiments on Caltech101-7} \label{sec:exp-cal}

We tested the overall performance of three compared methods on Caltech101-7
with all combinations of more than two views.
For six views, there are 42 combinations in total.
This experiment is conducted on 30\% training data and 70\%
testing data by running three compared methods on each combination separately
with the size of the common space (or rank) varied from 3 to 20.
The experiment is repeated 10 times on randomly splits drawn from the input data,
and the mean accuracies with standard deviations
of three compared methods are reported in Table \ref{tab:combination}.

\begin{table}
	\caption{Mean accuracy and standard deviation of three compared methods on $42$ data sets generated from Caltech101-7 over $10$ random splits with 30\% training data and rank $20$. }
	\label{tab:combination}
	\centering
	\begin{tabular}{lccc} \hline
		views & als & Alg.~\ref{Algorithm-TCCA} & mcca \\ \hline\hline
		centrist+gist+lbp & 95.23 $\pm$ 0.69 & \textbf{95.30 $\pm$ 0.73} & 90.74 $\pm$ 0.88\\
		centrist+gist+hog & 95.26 $\pm$ 0.58 & \textbf{95.32 $\pm$ 0.46} & 90.28 $\pm$ 0.96\\
		centrist+gist+ch & 92.47 $\pm$ 2.21 & \textbf{93.86 $\pm$ 0.96} & 90.66 $\pm$ 1.25\\
		centrist+gist+sift-spm & 95.56 $\pm$ 0.75 & \textbf{95.93 $\pm$ 0.39} & 92.59 $\pm$ 0.84\\
		centrist+lbp+hog & 94.95 $\pm$ 0.63 & \textbf{95.19 $\pm$ 0.65} & 90.09 $\pm$ 0.75\\
		centrist+lbp+ch & 92.63 $\pm$ 0.56 & \textbf{92.97 $\pm$ 0.65} & 90.06 $\pm$ 1.04\\
		centrist+lbp+sift-spm & 94.82 $\pm$ 0.75 & \textbf{95.15 $\pm$ 0.71} & 91.59 $\pm$ 0.71\\
		centrist+hog+ch & 91.29 $\pm$ 1.14 & \textbf{92.62 $\pm$ 1.07} & 89.79 $\pm$ 0.68\\
		centrist+hog+sift-spm & 93.46 $\pm$ 1.34 & \textbf{93.86 $\pm$ 1.27} & 91.15 $\pm$ 0.55\\
		centrist+ch+sift-spm & 90.24 $\pm$ 1.83 & \textbf{92.29 $\pm$ 0.94} & 88.99 $\pm$ 0.33\\
		gist+lbp+hog & 95.49 $\pm$ 0.92 & \textbf{95.63 $\pm$ 0.82} & 90.25 $\pm$ 0.85\\
		gist+lbp+ch & 93.04 $\pm$ 1.11 & \textbf{93.99 $\pm$ 0.60} & 90.93 $\pm$ 1.29\\
		gist+lbp+sift-spm & 95.71 $\pm$ 1.16 & \textbf{96.02 $\pm$ 0.60} & 92.41 $\pm$ 0.92\\
		gist+hog+ch & 90.86 $\pm$ 1.74 & \textbf{91.87 $\pm$ 1.92} & 89.60 $\pm$ 0.74\\
		gist+hog+sift-spm & 93.02 $\pm$ 0.58 & \textbf{93.05 $\pm$ 0.52} & 91.03 $\pm$ 0.47\\
		gist+ch+sift-spm & 90.14 $\pm$ 2.19 & \textbf{92.73 $\pm$ 1.09} & 90.25 $\pm$ 0.86\\
		lbp+hog+ch & 91.65 $\pm$ 1.51 & \textbf{92.98 $\pm$ 1.12} & 89.88 $\pm$ 0.67\\
		lbp+hog+sift-spm & 93.18 $\pm$ 0.77 & \textbf{94.16 $\pm$ 1.10} & 91.21 $\pm$ 0.50\\
		lbp+ch+sift-spm & 90.48 $\pm$ 1.61 & \textbf{92.33 $\pm$ 1.10} & 89.09 $\pm$ 0.46\\
		hog+ch+sift-spm & 88.35 $\pm$ 3.82 & \textbf{91.93 $\pm$ 0.94} & 90.07 $\pm$ 0.87\\			\hline\hline
		centrist+gist+lbp+hog & 92.14 $\pm$ 3.41 & \textbf{95.12 $\pm$ 1.03} & 89.14 $\pm$ 0.70\\
		centrist+gist+lbp+ch & 89.25 $\pm$ 2.64 & \textbf{92.41 $\pm$ 1.93} & 90.17 $\pm$ 0.87\\
		centrist+gist+lbp+sift-spm & 90.05 $\pm$ 3.10 & \textbf{94.60 $\pm$ 1.25} & 90.64 $\pm$ 0.88\\
		centrist+gist+hog+ch & 84.91 $\pm$ 4.33 & \textbf{93.01 $\pm$ 2.05} & 89.63 $\pm$ 0.53\\
		centrist+gist+hog+sift-spm & 88.32 $\pm$ 4.01 & \textbf{92.94 $\pm$ 0.81} & 90.42 $\pm$ 0.46\\
		centrist+gist+ch+sift-spm & 85.30 $\pm$ 3.53 & \textbf{92.90 $\pm$ 1.94} & 90.97 $\pm$ 0.63\\
		centrist+lbp+hog+ch & 87.09 $\pm$ 2.72 & \textbf{92.52 $\pm$ 1.78} & 89.29 $\pm$ 0.55\\
		centrist+lbp+hog+sift-spm & 87.00 $\pm$ 5.27 & \textbf{93.30 $\pm$ 2.36} & 89.82 $\pm$ 0.58\\
		centrist+lbp+ch+sift-spm & 84.64 $\pm$ 3.74 & \textbf{92.21 $\pm$ 1.96} & 89.21 $\pm$ 0.90\\
		centrist+hog+ch+sift-spm & 84.65 $\pm$ 3.19 & \textbf{92.31 $\pm$ 1.33} & 90.02 $\pm$ 0.51\\
		gist+lbp+hog+ch & 86.50 $\pm$ 6.35 & \textbf{92.76 $\pm$ 1.41} & 89.11 $\pm$ 0.38\\
		gist+lbp+hog+sift-spm & 87.96 $\pm$ 2.63 & \textbf{93.86 $\pm$ 1.21} & 89.96 $\pm$ 0.47\\
		gist+lbp+ch+sift-spm & 85.25 $\pm$ 3.50 & \textbf{91.93 $\pm$ 1.98} & 90.83 $\pm$ 0.60\\
		gist+hog+ch+sift-spm & 81.57 $\pm$ 5.98 & \textbf{90.82 $\pm$ 2.06} & 90.53 $\pm$ 0.54\\
		lbp+hog+ch+sift-spm & 83.20 $\pm$ 4.93 & \textbf{91.59 $\pm$ 1.65} & 89.89 $\pm$ 0.68\\
		
		\hline\hline
		
		gist+lbp+hog+ch+sift-spm & 84.38 $\pm$ 3.13 & \textbf{91.32 $\pm$ 2.46} & 89.76 $\pm$ 0.61\\
		centrist+lbp+hog+ch+sift-spm & 86.11 $\pm$ 3.49 & \textbf{91.80 $\pm$ 2.32} & 89.48 $\pm$ 0.73\\
		centrist+gist+hog+ch+sift-spm & 86.67 $\pm$ 4.70 & \textbf{92.51 $\pm$ 1.11} & 90.12 $\pm$ 0.60\\
		centrist+gist+lbp+ch+sift-spm & 88.70 $\pm$ 3.73 & \textbf{93.07 $\pm$ 1.35} & 89.96 $\pm$ 0.87\\
		centrist+gist+lbp+hog+sift-spm & 85.31 $\pm$ 4.66 & \textbf{94.17 $\pm$ 1.37} & 89.25 $\pm$ 0.47\\
		centrist+gist+lbp+hog+ch & 85.75 $\pm$ 4.66 & \textbf{92.48 $\pm$ 2.10} & 89.00 $\pm$ 0.49\\
		\hline\hline
		sift-spm+ch+hog+lbp+gist+centrist & 87.91 $\pm$ 2.98 & \textbf{90.09 $\pm$ 2.95} & 89.33 $\pm$ 0.65\\		
		\hline
	\end{tabular}
\end{table}

From Table \ref{tab:combination}, we have the following observations: (i) als outperforms mcca on three views, but underperforms mcca for more than three views; (ii) Our method outperforms both als and mcca consistently over all 42 combinations. These results imply that tensor-based methods can outperform mcca, when a good tensor approximation solver like our proposed algorithm is applied.

\begin{figure}[H]
	\centering
	\begin{tabular}{ccc}
		\includegraphics[width=0.31\textwidth]{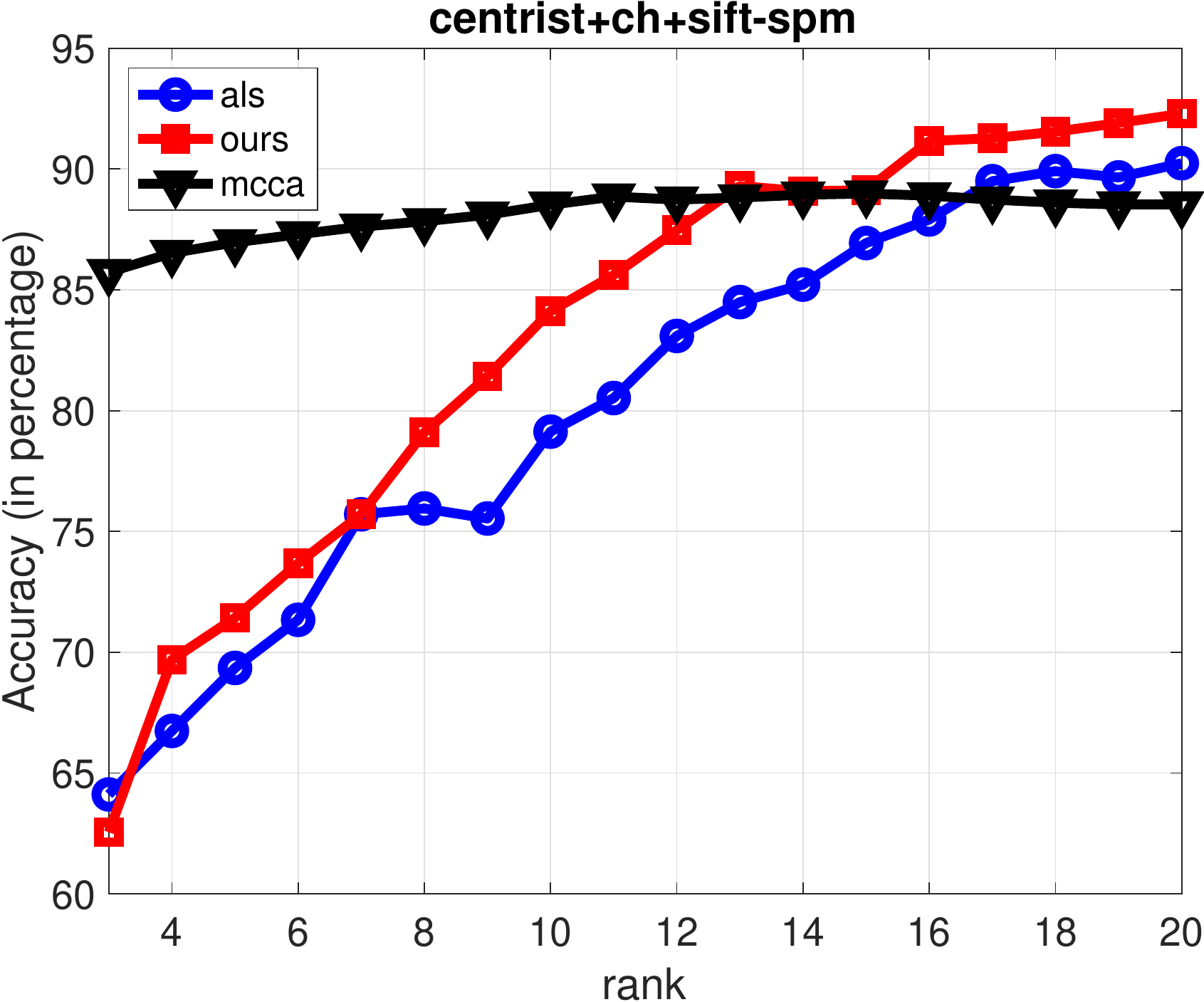}  &
		\includegraphics[width=0.31\textwidth]{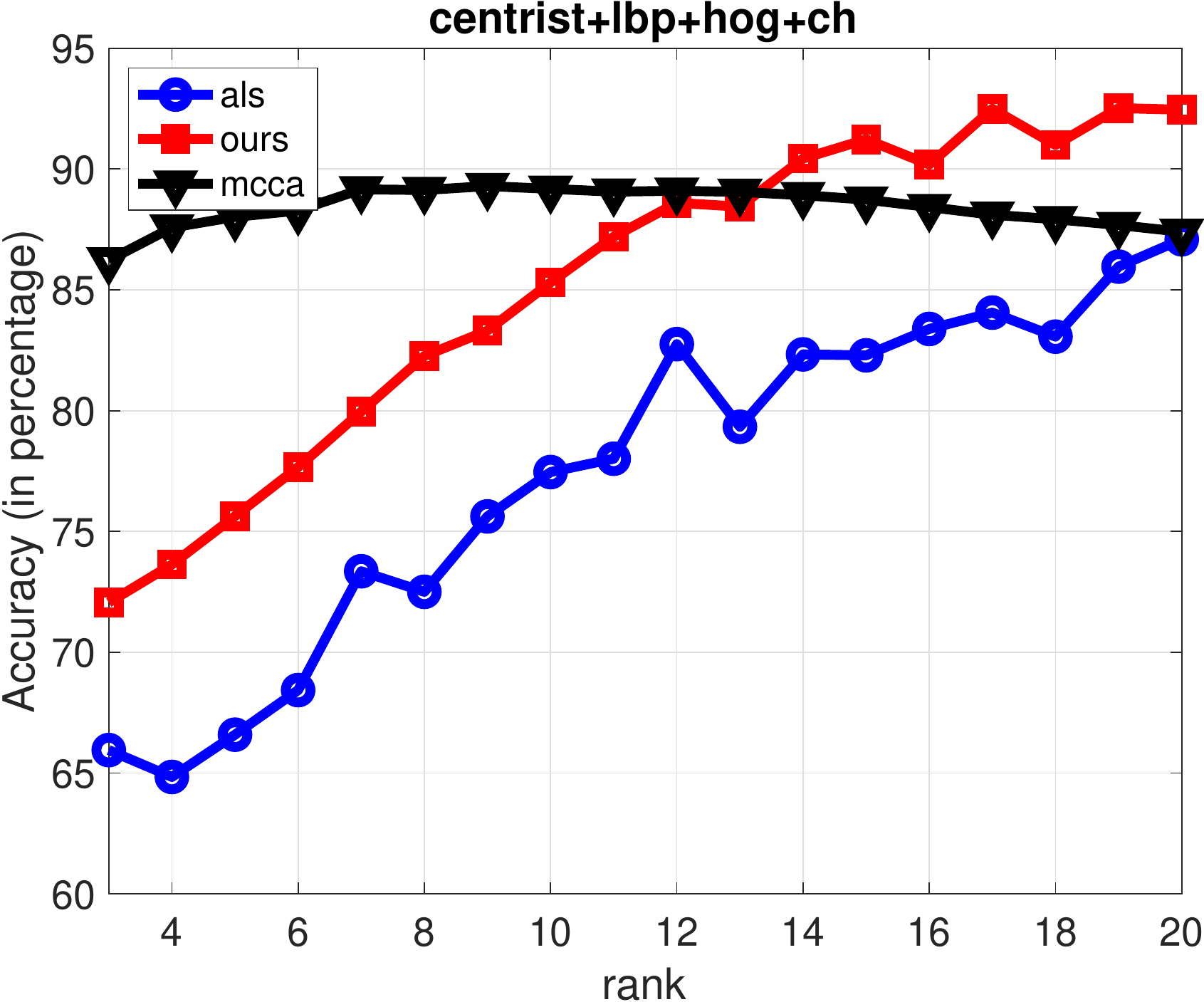} &		 \includegraphics[width=0.31\textwidth]{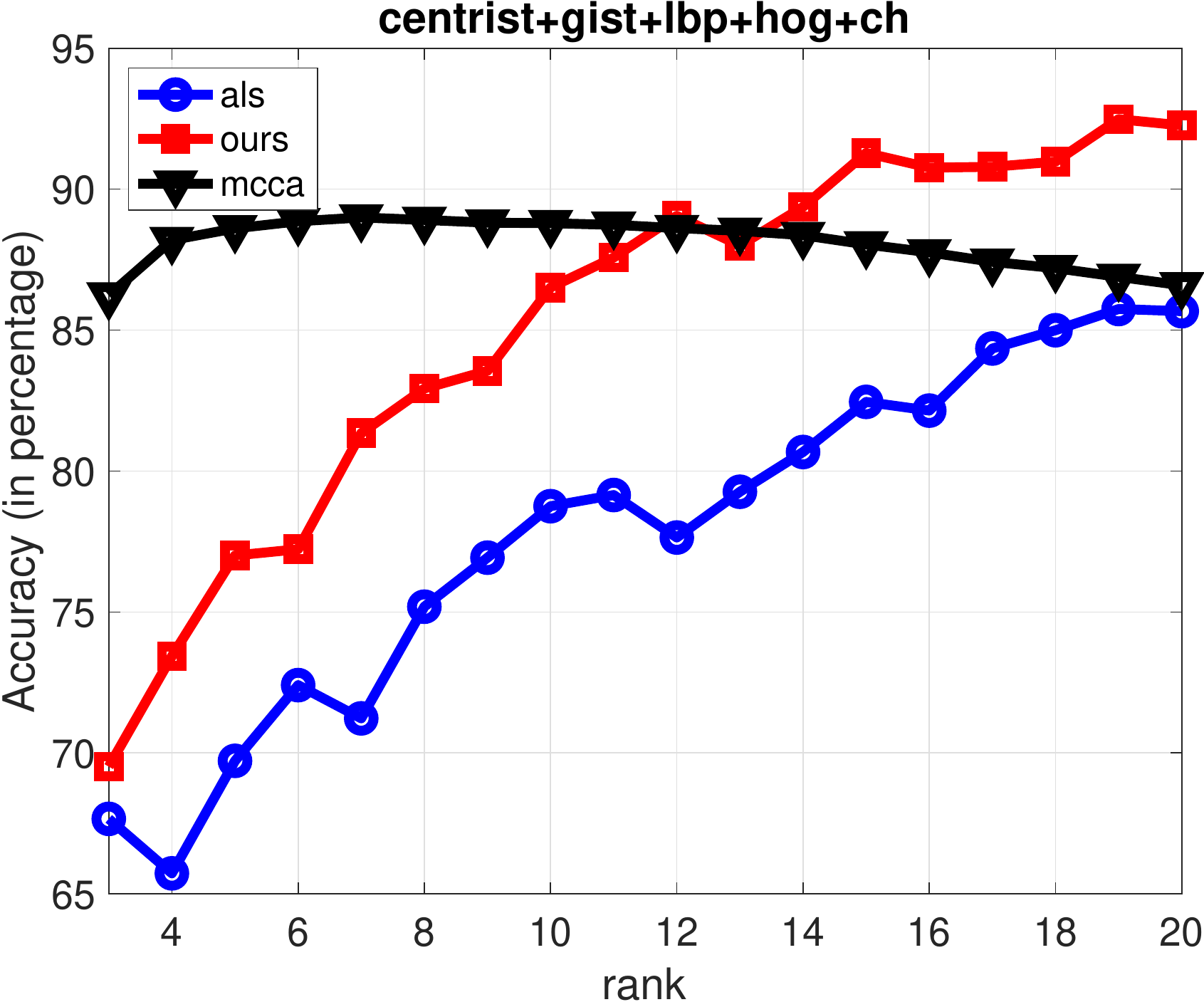} \\
		
		\includegraphics[width=0.31\textwidth]{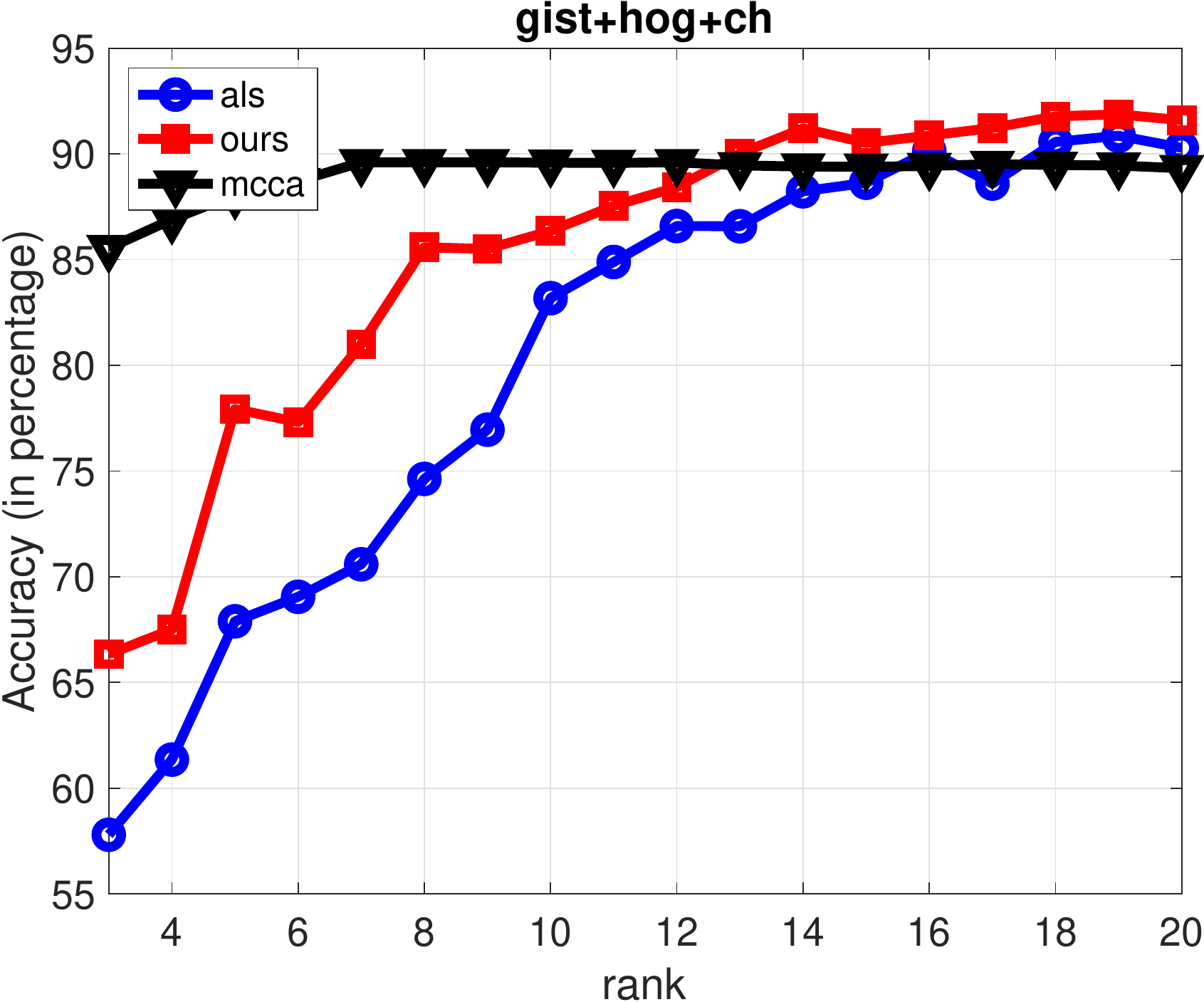}  &
		\includegraphics[width=0.31\textwidth]{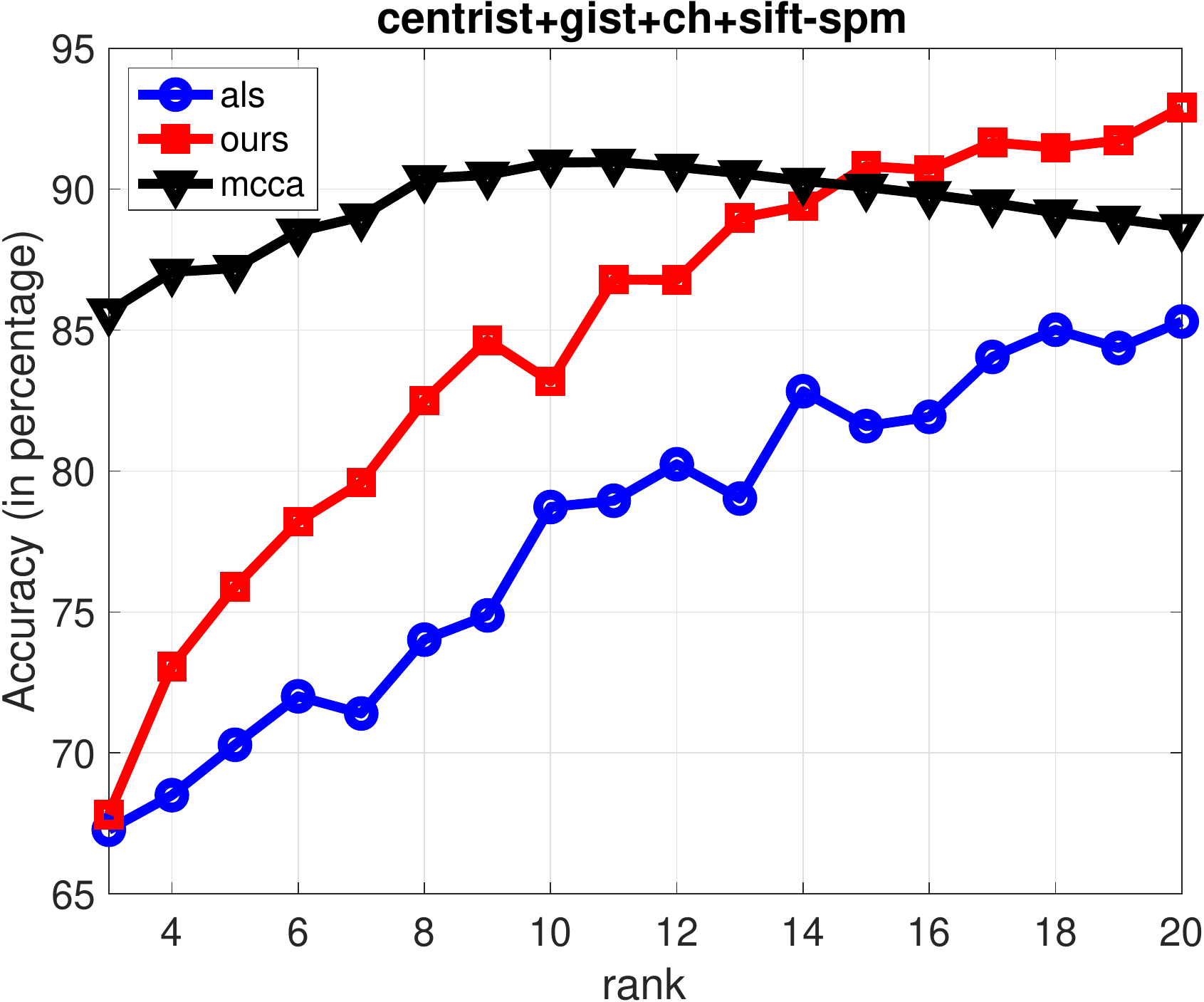}  &		 \includegraphics[width=0.31\textwidth]{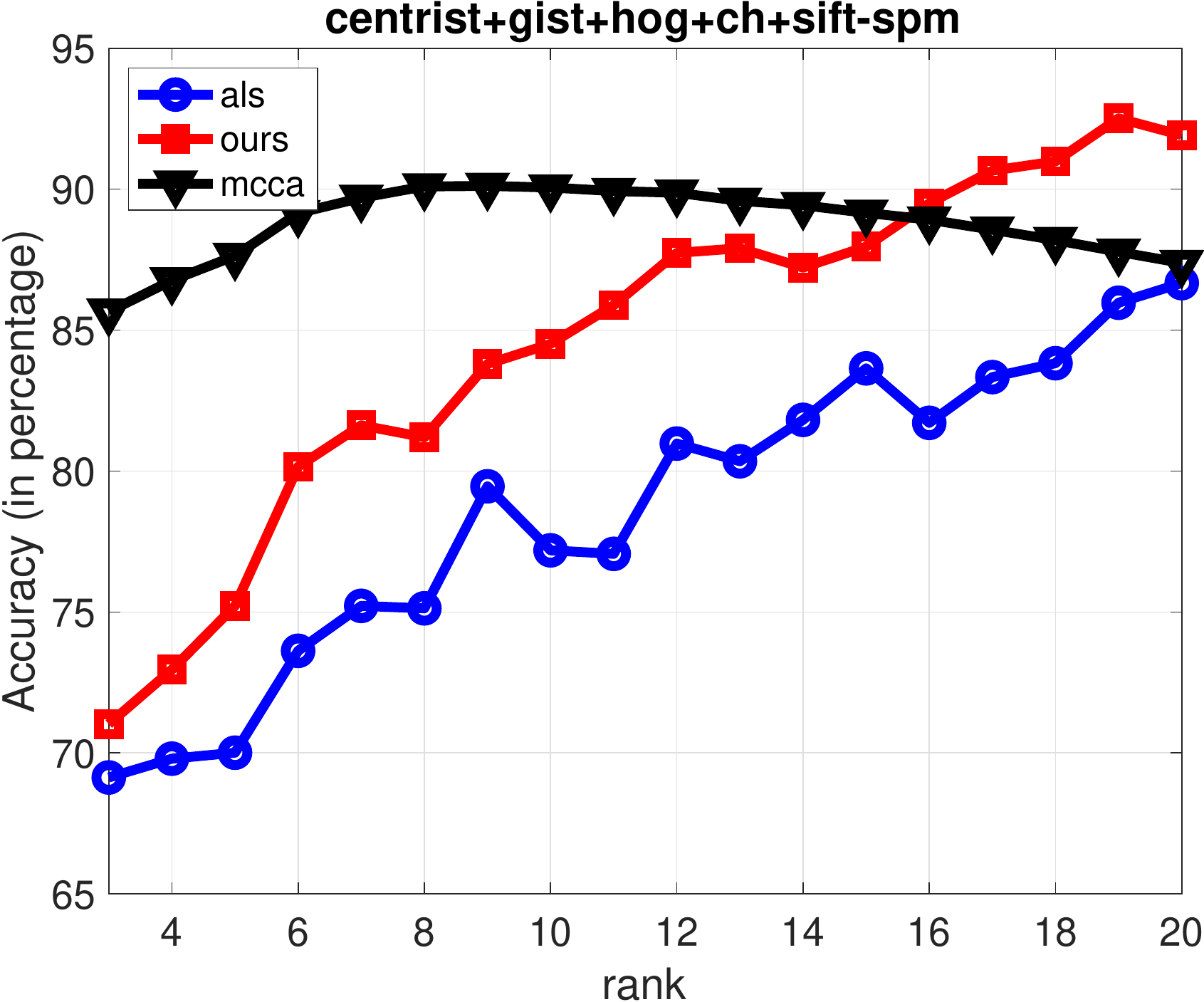} \\
	\end{tabular}
	\caption{Experimental results of compared methods on Caltech101-7 with three, four and five views over 30\% training data. } \label{fig:views}
\end{figure}

We further investigate the impact of compared methods
in terms of the varied ranks and the training ratios.
In Figure~\ref{fig:sense-caltech}(a),
the mean accuracy of testing data obtained by three compared methods varies
when the training ratio increases from 10\% to 70\%.
Due to the complexity of the tensor approximation problem,
both als and our method show larger fluctuations
than that of mcca when the training ratio increases.
However, our method consistently outperforms
both als and mcca over all tested training ratios.
In Figure~\ref{fig:sense-caltech}(b),
we show the mean accuracy of compared methods on 10 random splits with
30\% training data by varying rank from 3 to 20 on six views.
In addition, we show the mean accuracy of
compared methods on varied ranks on 30\%
with respect to combinations of different views in Figure \ref{fig:views}.
All these results demonstrate a similar trend with respect to testing accuracy
when the rank increases from 3 to 20: mcca shows better performance on small ranks,
but our method outperforms both mcca and als on large ranks, and overall our method obtains the best performance over all tested ranks.
From Figure~\ref{fig:sense-caltech}, we can see that our model on 30\%
training data shows the worst results
comparing to other training ratios.
This implies that the results in Figure \ref{fig:sense-caltech}
show the worst results of our method,
which still outperforms
the other two methods as shown in Figure \ref{fig:views}.
{Moreover, we report the empirical comparison of computational time 
	for these three methods, according to values of
	ranks and views on Caltech101-7.
	The comparison is shown in Table~\ref{tab:time}.
	For cleanness, the average CPU time over the combinations 
	of fixed numbers of views is reported.
	As the experiments show, mcca is the fastest one, 
	because the generalized eigenvalue decomposition
	for matrices of  size $20 \times v$ ($v \in \{3, 4, 5, 6\}$)
	can be very fast. Our method is slower than mcca and als.
	As the number of views increases,
	the size of tensor increases and our method becomes slower.}
\begin{table}[!h]
	\caption{The CPU time of three compared methods on Caltech101-7 in terms of both ranks and views.}
	\label{tab:time}
	\centering
	\begin{small}
		\begin{tabular}{@{}c|ccc|ccc|ccc@{}} \hline
			\multirow{2}{*}{rank}  & \multicolumn{3}{c}{3 views} &
			\multicolumn{3}{c}{4 views} & \multicolumn{3}{c}{5 views}  \\ \cline{2-10}
			& mcca & als & Alg.~\ref{Algorithm-TCCA} & mcca & als & Alg.~\ref{Algorithm-TCCA}
			& mcca & als & Alg.~\ref{Algorithm-TCCA} \\  \hline
			3 & 0.0016& 0.0464& 0.0578& 0.0022& 0.0520& 0.0863& 0.0040& 0.2463& 0.7044 \\
			4 & 0.0013& 0.0450& 0.0656& 0.0022& 0.0517& 0.0916& 0.0037& 0.2388& 0.8510 \\
			5 & 0.0014& 0.0464& 0.0713& 0.0021& 0.0537& 0.1216& 0.0040& 0.2434& 1.0633 \\
			6 & 0.0013& 0.0459& 0.0781& 0.0021& 0.0522& 0.1431& 0.0041& 0.2461& 1.2318 \\
			7 & 0.0013& 0.0464& 0.0797& 0.0022& 0.0537& 0.1519& 0.0036& 0.2435& 1.3059 \\
			8 & 0.0013& 0.0478& 0.0843& 0.0020& 0.0520& 0.1578& 0.0037& 0.2563& 1.3738 \\
			9 & 0.0014& 0.0454& 0.0942& 0.0022& 0.0536& 0.1665& 0.0035& 0.2556& 1.4235 \\
			10 & 0.0013& 0.0461& 0.0961& 0.0023& 0.0528& 0.1778& 0.0036& 0.2609& 1.6305 \\ \hline
		\end{tabular}
	\end{small}
\end{table}

\subsection{Experiments on Scene15}

The experiments  same as in section \ref{sec:exp-cal} are performed on data Scene15. In Table \ref{tab:combination-2}, tensor-based methods including both als and ours outperform mcca on all view combinations. Our method outperforms als on three and four views, while it is competitive to als on five views. Figure \ref{fig:sense-scene15} demonstrates the sensitivity of compared methods by varying the rank and the training ratios. On data Scene15, the tensor-based methods are consistently better than mcca over all tested training ratios, while our method outperforms als on large ranks and is competitive on small ranks. These results are consistent with the observations on Caltech101-7 in section \ref{sec:exp-cal}, especially on relatively large ranks.

\begin{table}[t]
	\caption{Mean accuracy and standard deviation of two compared methods on $16$ data sets generated from Scene15 over $10$ random splits with 30\% training data and rank $20$. } \label{tab:combination-2}
	\centering
	\begin{tabular}{lccc}
		\hline
		views & als & ours & mcca \\
		\hline\hline		
		centrist+gist+lbp & 65.91 $\pm$ 1.45 & \textbf{66.66 $\pm$ 1.12} & 57.01 $\pm$ 0.81\\
		centrist+gist+hog & 67.21 $\pm$ 1.15 & \textbf{67.73 $\pm$ 1.08} & 58.47 $\pm$ 0.88\\
		centrist+gist+sift-spm & 70.40 $\pm$ 1.48 & \textbf{72.68 $\pm$ 1.70} & 64.34 $\pm$ 1.65\\
		centrist+lbp+hog & 60.29 $\pm$ 1.80 & \textbf{62.32 $\pm$ 1.56} & 53.63 $\pm$ 0.72\\
		centrist+lbp+sift-spm & 60.30 $\pm$ 1.83 & \textbf{65.43 $\pm$ 1.45} & 58.10 $\pm$ 1.46\\
		centrist+hog+sift-spm & 63.05 $\pm$ 1.98 & \textbf{67.45 $\pm$ 1.24} & 58.11 $\pm$ 1.24\\
		gist+lbp+hog & 61.08 $\pm$ 1.33 & \textbf{62.86 $\pm$ 1.20} & 54.20 $\pm$ 0.74\\
		gist+lbp+sift-spm & 65.82 $\pm$ 1.68 & \textbf{68.88 $\pm$ 1.30} & 59.23 $\pm$ 1.42\\
		gist+hog+sift-spm & 63.12 $\pm$ 3.94 & \textbf{67.13 $\pm$ 2.28} & 54.71 $\pm$ 1.76\\
		lbp+hog+sift-spm & 57.08 $\pm$ 2.54 & \textbf{60.58 $\pm$ 1.32} & 51.01 $\pm$ 2.05\\
		
		\hline\hline
		centrist+gist+lbp+hog & 62.33 $\pm$ 2.21 & \textbf{63.83 $\pm$ 2.11} & 51.07 $\pm$ 0.72\\
		centrist+gist+lbp+sift-spm & 61.96 $\pm$ 3.02 & \textbf{65.34 $\pm$ 1.62} & 55.56 $\pm$ 1.08\\
		centrist+gist+hog+sift-spm & 62.41 $\pm$ 3.08 & \textbf{63.67 $\pm$ 3.70} & 58.69 $\pm$ 1.19\\
		centrist+lbp+hog+sift-spm & 54.63 $\pm$ 3.26 & \textbf{57.98 $\pm$ 2.37} & 51.77 $\pm$ 1.57\\
		gist+lbp+hog+sift-spm & 58.93 $\pm$ 3.18 & \textbf{61.28 $\pm$ 2.26} & 50.05 $\pm$ 1.15\\
		
		\hline\hline
		centrist+gist+lbp+hog+sift-spm & \textbf{60.72 $\pm$ 2.87} & 60.35 $\pm$ 3.33 & 49.11 $\pm$ 1.49\\		
		\hline
	\end{tabular}
\end{table}

\begin{figure}[t]
	\centering
	\begin{tabular}{ccc}
		\includegraphics[width=0.3\textwidth]{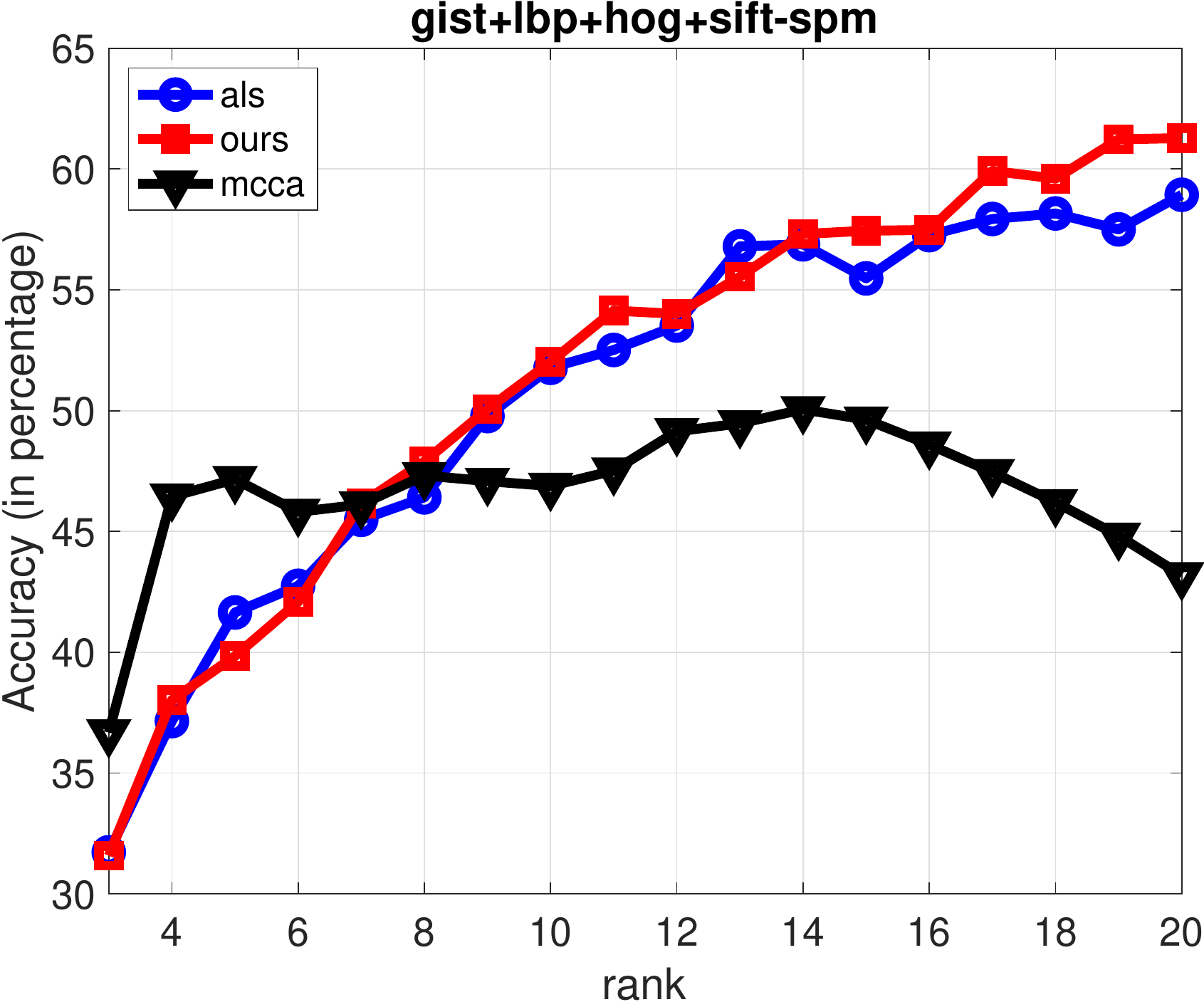}  & \includegraphics[width=0.3\textwidth]{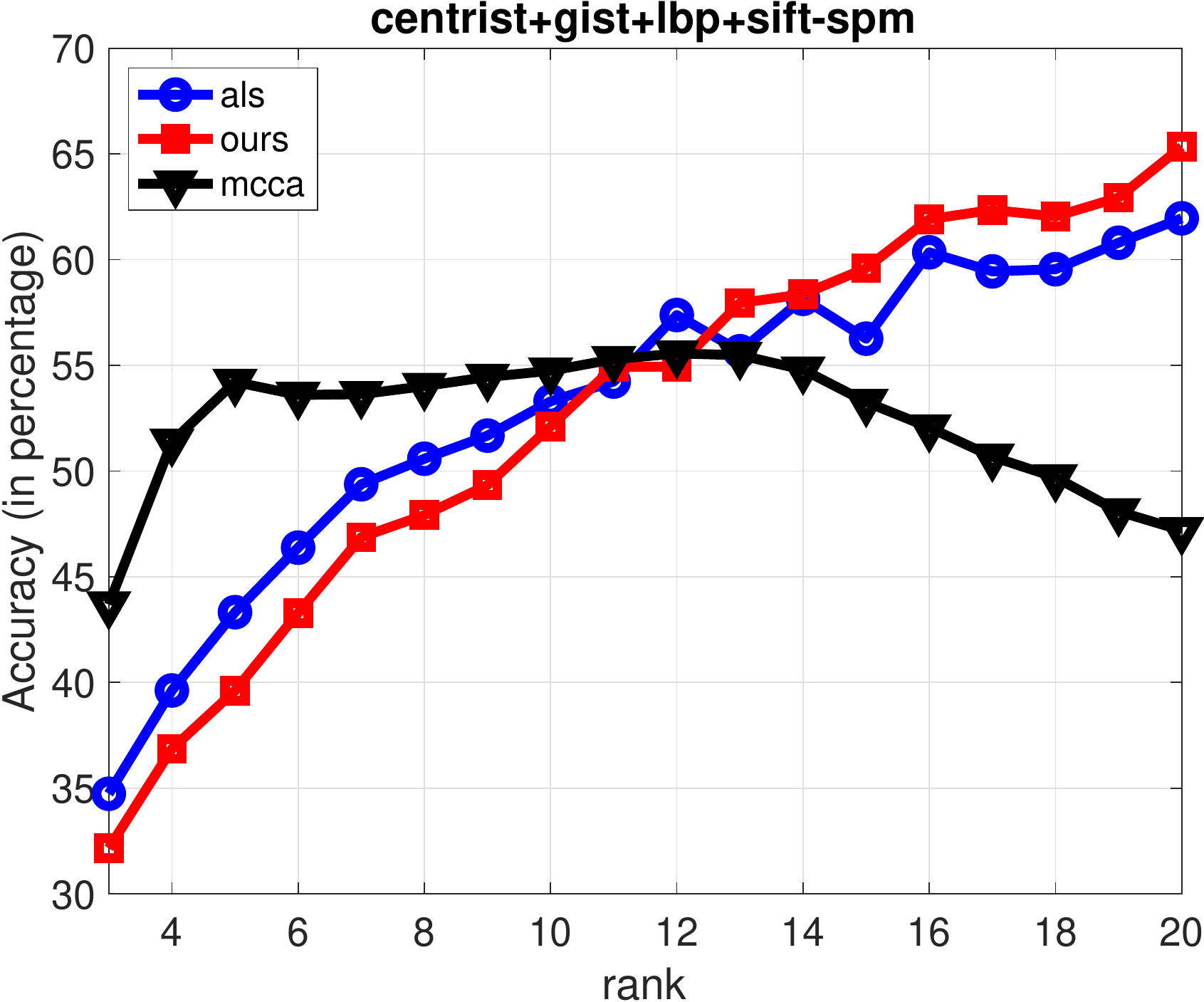} & \includegraphics[width=0.3\textwidth]{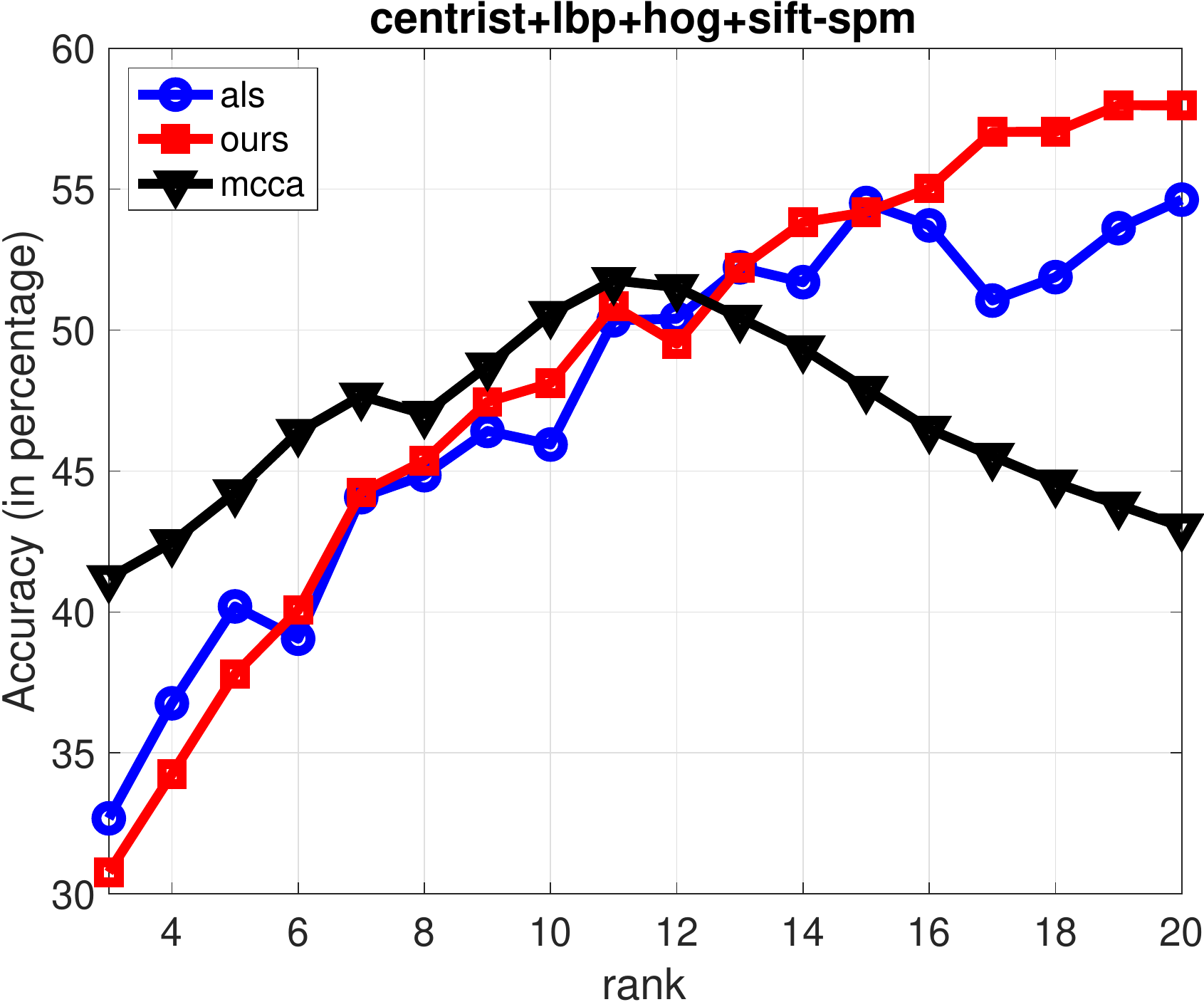}\\
		\includegraphics[width=0.3\textwidth]{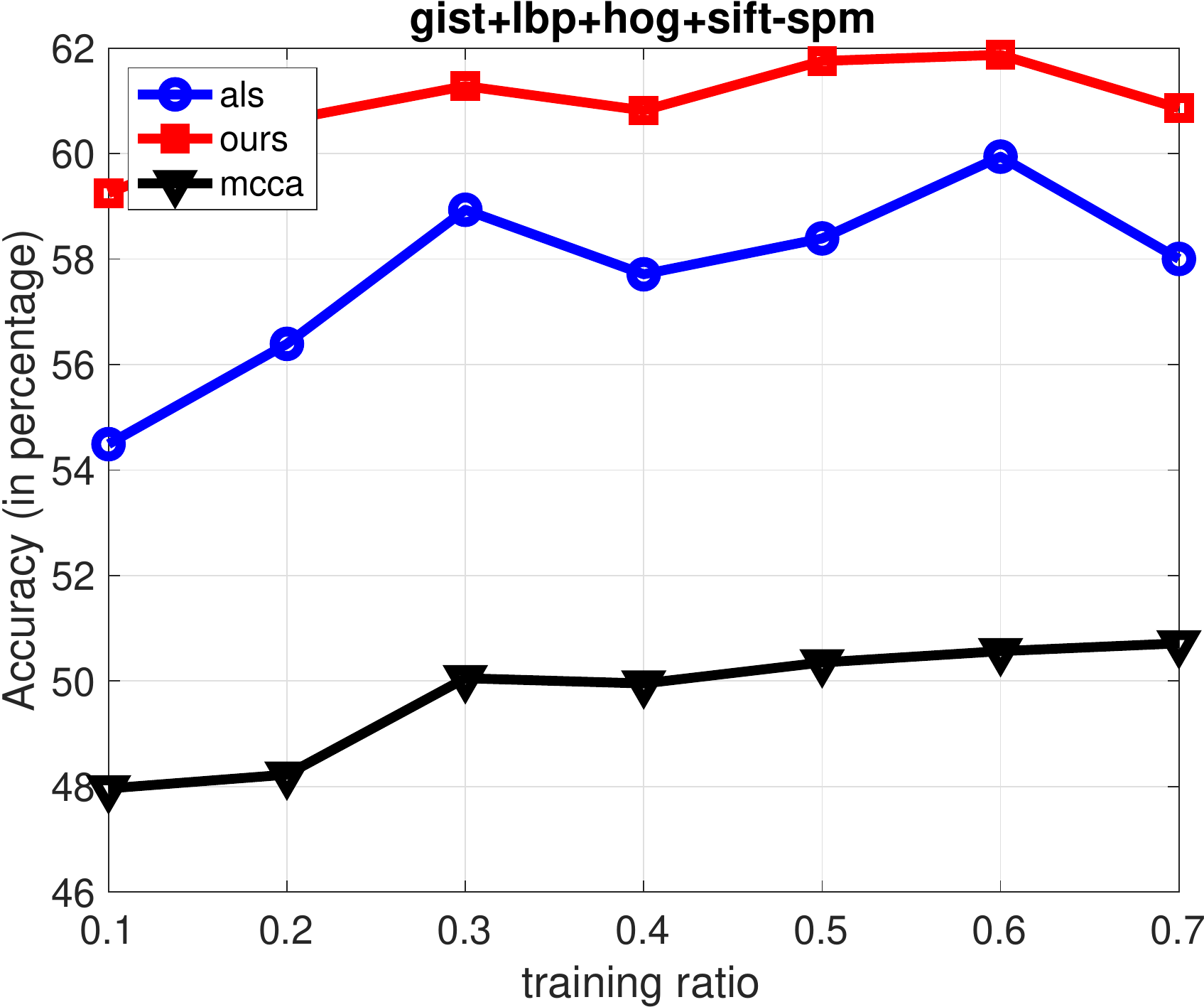} &
		\includegraphics[width=0.3\textwidth]{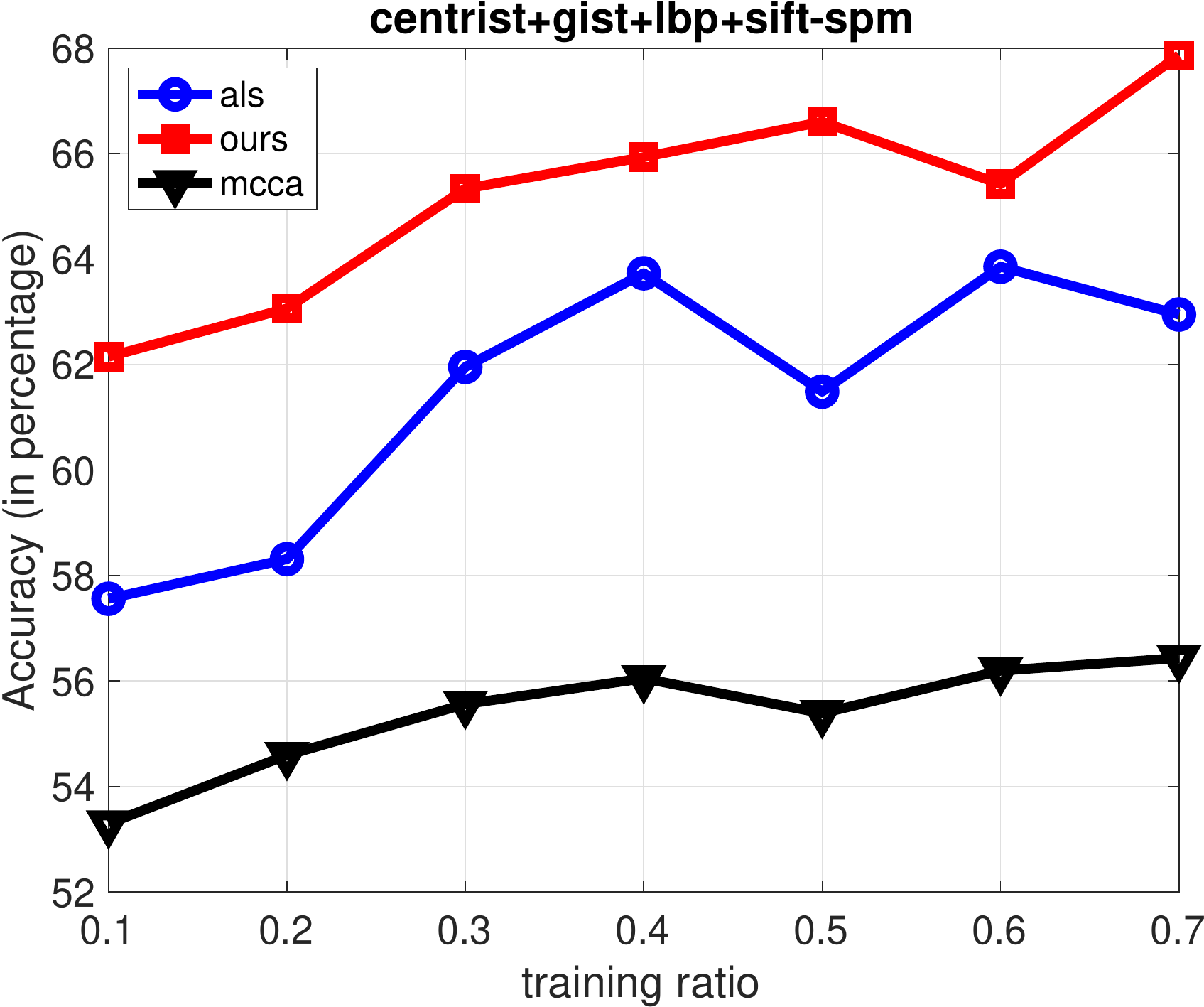} &
		\includegraphics[width=0.3\textwidth]{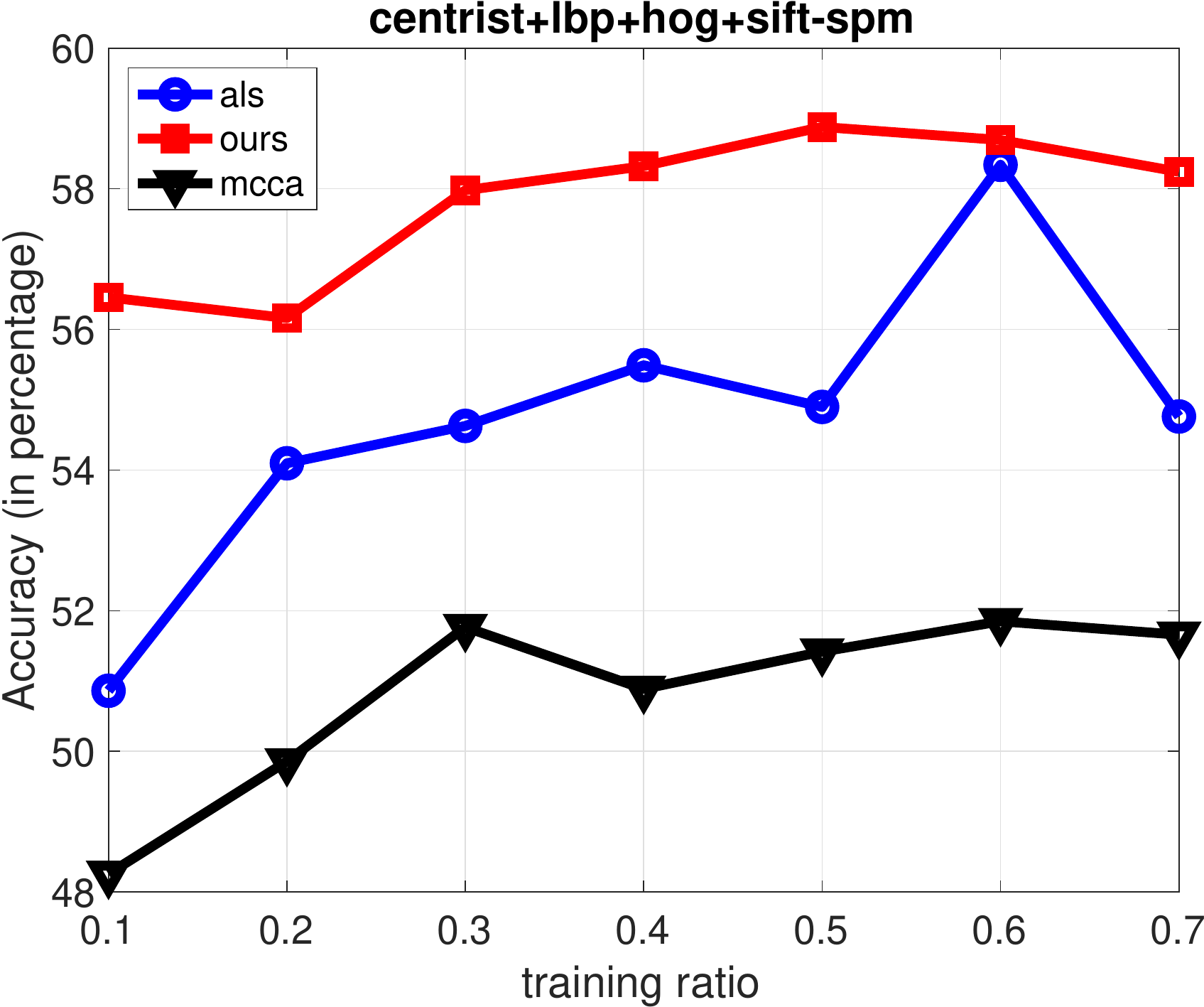}\\
	\end{tabular}
	\caption{Sensitivity analysis of compared methods with four views on data Scene15. (top row) varying the size of common space on 30\% training data; (bottom row) varying the training ratios over common spaces from 3 to 20.} \label{fig:sense-scene15}
\end{figure}

\bigskip \noindent
{\bf Acknowledgement}
The first author is partially supported by the NSF grant
DMS-2110780. The second author is partially supported by the NSF grant DMS-2009689.


\end{document}